\documentclass{article}

\usepackage{fullpage,hyperref,authblk,natbib}
\usepackage[title]{appendix}
\usepackage{microtype}
\usepackage{graphicx}
\usepackage{subfigure}
\usepackage{booktabs} 

\usepackage{epsfig}
\usepackage{graphicx}
\usepackage{amsmath}
\usepackage{amssymb}
\usepackage{color}
\usepackage{amsthm}
\usepackage{todonotes}
\usepackage{hyperref}
\usepackage{wasysym}

\usepackage{amsmath,amsfonts,color}

\newcommand{\ceil}[1]{\left \lceil{#1}\right \rceil }

\newcommand{\R}{\mathbb R}

\newcommand{\mB}{\mathcal B}
\newcommand{\mQ}{\mathcal Q}

\newcommand{\mS}{\mathcal S}

\newcommand{\mP}{\mathcal P}

\newcommand{\mN}{\mathcal N}

\newcommand{\tr}{\mathbf{tr}}

\newcommand{\dom}{\mathbf{dom}}

\newcommand{\LMO}{\mathbf{LMO}}
\newcommand{\conv}{\mathbf{conv}}

\newcommand{\res}{\mathbf{res}}

\newcommand{\supp}{\mathbf{supp}}

\newcommand{\range}{\mathbf{range}}

\newcommand{\mb}[1]{\mathbf{#1} }

\newcommand{\minimize}[1]{\underset{#1}{\mathrm{minimize}}}
\newcommand{\maximize}[1]{\underset{#1}{\mathrm{maximize}}}

\newcommand{\argmin}[1]{\underset{#1}{\mathrm{argmin}}}

\newcommand{\argmax}[1]{\underset{#1}{\mathrm{argmax}}}

\newcommand{\sign}{\mathbf{sign}}
\newcommand{\gap}{\mathbf{gap}}

\newcommand{\inte}{\mathbf{int}}

\newcommand{\diam}{\mathbf{diam}}
\newcommand{\ri}{\mathbf{ri}}

\newcommand{\bmat}{\left[\begin{matrix}}
\newcommand{\emat}{\end{matrix}\right]}

\newtheorem{assumption}{Assumption}
\newtheorem{corollary}{Corollary}

\newtheorem{property}{Property}
\newtheorem{lemma}{Lemma}
\newtheorem{theorem}{Theorem}

\newtheorem{innercustomprop}{Property}
\newenvironment{customprop}[1]
  {\renewcommand\theinnercustomprop{#1}\innercustomprop}
  {\endinnercustomprop}

\newtheorem{innercustomthm}{Theorem}
\newenvironment{customthm}[1]
  {\renewcommand\theinnercustomthm{#1}\innercustomthm}
  {\endinnercustomthm}

\newtheorem{innercustomlem}{Lemma}
\newenvironment{customlem}[1]
  {\renewcommand\theinnercustomlem{#1}\innercustomlem}
  {\endinnercustomlem}

\theoremstyle{definition}
\newtheorem{definition}{Definition}

\theoremstyle{remark}
\newtheorem*{remark}{Remark}

\newcommand*\samethanks[1][\value{footnote}]{\footnotemark[#1]}

\begin{document}

\title{Safe Screening  for the Generalized Conditional Gradient Method}

\author[1]{Yifan Sun\thanks{
This work was funded in part by the French government under management of Agence Nationale de la Recherche as part of the ``Investissements d'avenir" program, reference ANR-19-P3IA-0001 (PRAIRIE 3IA Institute). We also acknowledge
support the European Research Council (grant SEQUOIA 724063).}
\thanks{This author is also funded in part  by AXA pour la recherche and Kamet Ventures, as well as a Google focused award.}
\thanks{email: yifan.sun@inria.fr}}
\author[1]{Francis Bach  \samethanks[1]
 \thanks{email: francis.bach@inria.fr}}
\affil[1]{INRIA - D\'epartement d'Informatique de l'Ecole Normale Sup\'erieure
PSL Research University Paris, France}
\date{February 2020}


\maketitle

\begin{abstract}
The conditional gradient method (CGM) has been widely used for fast sparse approximation, having a low per iteration computational cost for structured sparse regularizers. We explore the sparsity acquiring properties of a generalized CGM (gCGM), where the constraint is replaced by a penalty function based on a gauge penalty; this can be done without significantly increasing the per-iteration computation, and applies to general notions of sparsity. 
Without assuming bounded iterates, we show $O(1/t)$ convergence of the function values and gap of gCGM. We couple this with a safe screening rule, and show that at a rate $O(1/(t\delta^2))$, the screened support matches the support at the solution, where $\delta \geq 0$ measures how close the problem is to being degenerate.
In our experiments, we show that the gCGM for these modified penalties have similar feature selection properties as common penalties, but with potentially more stability over the choice of hyperparameter. 
\end{abstract}

\section{Introduction}
The conditional gradient method (CGM)  is an iterative method with a particularly cheap per-iteration cost, and is thus favored in large-scale machine  learning applications. 
A generalized  CGM (gCGM) minimizes over the regularized convex problem 
\begin{equation}
\minimize{x\in \R^d}\; f(x) + h(x),
\label{eq:main0}
\end{equation}
where $f$ is a smooth convex function and $h$ promotes structural properties. 
At each iteration, the method updates the primal variable $x^{(t)}$  as
\begin{eqnarray}
s^{(t)} &=& \argmin{s \in \R^d}\; \nabla f(x^{(t)})^Ts + h(s)\label{eq:step:genlmo}\\
x^{(t+1)} &=& (1-\theta^{(t)}) x^{(t)} + \theta^{(t)} s^{(t)},\label{eq:step:update}
\end{eqnarray}
where $\theta^{(t)} \in [0,1]$ is a pre-determined decaying sequence. 
If $h = \iota_\mP$ the indicator function for a compact convex set~$\mP$, then this iteration scheme reduces to the vanilla CGM (vCGM) for constrained optimization; the main extension in gCGM is to solve unconstrained (but penalized) problems, where the iterates are not forced to stay within a specified bounded set. 
Specifically, 
we consider $h(x) = \phi(\kappa_\mP(x))$, where $\phi:\R_+\to\R$  is a monotonically nondecreasing function and $\kappa_\mP$ is the  gauge penalty function induced by ``nice'' sets $\mP$; overall, this penalty encourages sparsity in the minimizer $x^*$ with respect to the extremal vertices of $\mP$. 

\subsection{Related work}
\paragraph{Applications.}
A main use case of CGMs is in finding generalized sparse solutions to convex losses \cite{jaggi2013revisiting,chandrasekaran2012convex}, where the $\ell_1$-norm penalty in promoting element-wise sparsity \cite{tibshirani1996regression, donoho2006compressed,candes2005decoding,candes2006robust} is generalized to gauge functions $\kappa_\mP$ that promote sparsity with respect to  ``atoms", which are the lowest dimensional facets of a convex set $\mP$. 
This generalizes sparse optimization to applications such as  low-rank matrix optimization \cite{yu2017generalized,freund2017extended} and grouped feature extraction \cite{vinyes2017fast,zeng2014ordered,bondell2008simultaneous}.
Additionally, these atoms may be feasible solutions to combinatorial problems, and \eqref{eq:main0} may be a convex relaxation, such as in submodular optimization \cite{bach2010structured} and  object tracking \cite{chari2015pairwise}. 
Other machine learning applications involving the CGM include  
 graphical models \cite{krishnan2015barrier},
multitask learning \cite{sener2018multi}, 
SVMs  \cite{lacoste2012block}, particle filtering \cite{lacoste2015sequential},  
and deep learning \cite{ping2016learning, berrada2018deep}.

\paragraph{Safe screening.}  Safe screening rules for LASSO were first proposed by \citet{ghaoui2010safe}, and have since been extended to a number of smooth losses and generalized penalties \cite{fercoq2015mind, xiang2012fast, wang2014safe, liu2013safe,malti2016safe,raj2016screening,ndiaye2015gap,wangjie,bonnefoyantoine,zhouzhao}.  Rules for ``group'' testing \cite{herzet2018joint} and sample screening \cite{shibagaki2016simultaneous, ogawa2013safe} have also been considered. An interesting related work is the ``stingy coordinate descent'' method \cite{johnson2017stingy} for LASSO, which optimizes the sparse regularized problem in a CGM-like manner, but uses screening to dynamically skip steps; this kind of methods can be extended to gCGM as well for generalized atoms.

A key challenge in penalized sparsity problems is that when the dual is constrained, the corresponding dual variable  may not be feasible, and thus the computed gap is $+\infty$. In this context, gap-safe screening  methods offer a number of solutions, such as scaling or projecting to acquire a dual feasible candidate. We do not attempt to remedy this problem; in fact, in gCGM, the typical LASSO penalty presents a fundamental implementation issue, in that if $h(x) = \|x\|_1$, then the problem \eqref{eq:step:genlmo} can easily  be unbounded. By requiring  curvature of $\phi(\xi)$ for large enough $\xi$, we  ensure that the dual problem is unbounded, and the natural dual candidate $z^{(t)} = -\nabla f(x^{(t)})$ does not need to be adjusted to ensure bounded subproblems \eqref{eq:step:genlmo}.
This ensures that gCGM is well-defined and converging to the solution; additionally, it allows easier gap calculations.
These curvature conditions will be elaborated in later sections.   

\paragraph{Conditional gradient methods.}  
The vCGM, also called the Frank-Wolfe method \cite{frank1956algorithm,dunn1978conditional}, considers minimizing \eqref{eq:main0} as a constrained optimization problem (where $h(x) =\iota_{C\mP}(x)$ for some scaling $C > 0$). The method is particularly  useful when computing the supporting hyperplane in \eqref{eq:step:genlmo} is computationally simple (e.g., when $\mP$ is the unit ball of the $\ell_1$-norm or the nuclear norm). 
Thus, CGM is widely considered in the context of generalized sparse optimization
\cite{hazan2008sparse,clarkson2010coresets,jaggi2013revisiting,tewari2011greedy}, with many variations such as backward steps \cite{lacostejulienjaggi, nikhilforwardbackward} and fully-corrective steps \cite{vonhohenbalken}, and connections to other methods like mirror descent \cite{bach2015duality}, cutting plane method \cite{limitedkelley}, and greedy coordinate-wise methods \cite{clarkson2010coresets}.

In comparison, gCGM (where $h(x)$ may be unconstrained) has been much less studied, and has appeared under different names, like regularized coordinate minimization \cite{dudik2012lifted}. 
An $O(1/t)$ convergence rate has been shown for specific smooth functions \cite{mu2016scalable},  with bounded assumptions on iterates  \cite{bach2015duality}, or with improvement steps to ensure boundedness of sublevel sets \cite{yu2017generalized,harchaoui2015conditional}.
When $f$ is quadratic and for a special form of $\phi$, the gCGM can be shown to be equivalent to a form of the iterative shrinkage method, and under proper problem conditioning, has linear convergence
\cite{bredies2009generalized,bredies2008iterated}.
We also give an $O(1/t)$ convergence rate on objective function values and minimum gap convergence, but relinquish any assumption on boundedness of iterates.


\subsection{Contributions}
 We analyze the \emph{convergence} and \emph{support recovery} properties of the gCGM for \eqref{eq:main0}, where $h(x) = \phi(\kappa_\mP(x))$ involves only modifications of a gauge function $\kappa_\mP(x)$.
 We assume that the loss function $f$ is $L$-smooth, the function $\phi(\xi)$ grows at least quadratically when $\xi$ is large, and the set $\mP$ is convex and compact. Our contribution is threefold.
\begin{itemize}
\item  \emph{Without boundedness assumptions on iterates},  the function value error and minimum duality gap of  gCGM converge as $O(1/t)$. 

\item We provide a safe dual screening rule for any intermediate variable $x$. This rule is algorithmically agnostic, and generalizes SAFE screening rules for LASSO to any gauge function and any case where $\phi$ is monotonicaly nondecreasing, in particular to cases where the dual is unconstrained and thus always feasible.

\item Finally, by bounding the gradient error with the gap, we give a mechanism for deriving manifold identification rates for any version of gCGM where minimum gap rates are known. 
\end{itemize}
Additionally, our proof technique is from a convex analysis viewpoint, in that we measure all distances and errors in terms of gauges and support functions of $\mP$ (sometimes symmetrized).  
This is done for two reasons: first, to ensure that all analysis is linear invariant (in similar spirit as \citet{lacostejulienjaggi}); and second, for increased interpretability, as connections can be drawn to the much more intuitive (but restrictive) case of  $\kappa_\mP = \|\cdot\|_1$ in sparse optimization (and more commonly considered in screening literature). 
All proofs are given in the appendix.

\section{Preliminaries}
\label{sec:prelim}

\subsection{Generalized sparse optimization}

Define a finite set of points $\mP_0 = \{p_1,...,p_m\}\subset \R^d$, and its convex hull as $\mP = \conv(\mP_0)$; since $m$ is finite, $\mP$ is a  convex and compact set. 
We consider problems of the form
\begin{equation}
\minimize{x\in\R^d} \; f(x) + \phi(\kappa_\mP(x)),
\label{eq:main}
\end{equation}
where $\phi:\R_+\to\R_+$ is a  monotonically nondecreasing function. 
The function 
\begin{equation}
\kappa_\mP(x) = \min_{c_i\geq 0}\left\{\sum_{i=1}^m c_i : \sum_{i=1}^m c_ip_i = x\right\}
\label{eq:gauge}
\end{equation}
is the \emph{gauge function} of $\mP$; in particular, it measures the ``size" of $x$ by giving how much the set $\mP$ must be expanded (or can be contracted) to include $x$, and generalizes norms to any positive homogenous and subadditive function \cite{freund1987dual,chandrasekaran2012convex}. 
We define the \emph{support of $x$ with respect to $\mP_0$} (denoted $\supp_\mP(x)$) as the set of $p_i$'s in \eqref{eq:gauge} for which $c_i > 0$. Such a set may not be uniquely defined, but we consider \emph{support recovery} achieved if one such set is revealed.

Gauge functions can be seen as generalized versions of the $\ell_1$-norm, which is a convex promoter of nonzero vector sparsity.
In particular, if $\mP_0 = \{\pm e_k\}_{k=1}^d$ is the signed standard basis, then we exactly recover $\kappa_\mP(x) = \|x\|_1$. 
More generally, if $\mP_0$ contains $d$ vectors spanning $\R^d$, then defining the matrix $P = (p_1,...,p_d)$,  $\kappa_\mP(x) = \|P^{-1} x\|_1$, and promotes  vectors $x = Pc$ whose pre-image $c$ is sparse. But gauges also encompass more general scenarios, such as seminorms (e.g., total variation norm), non-polyhedral norms (e.g., nuclear norm), and conic constraints; they can also be manipulated to include ordering, such as with the OWL norm \citep{zeng2014ordered}, and discover groupings with the OSCAR norm \citep{bondell2008simultaneous}.

A ``dual gauge'' can be constructed as the support function 
\begin{equation}
\sigma_\mP(z) := \sup_{s\in \mP} z^Ts.
\label{eq:supportfn}
\end{equation}
In particular, if $\kappa_\mP$ is a norm, then $\sigma_\mP$ is the usual dual norm. Finding an optimal variable  $s$ in \eqref{eq:supportfn} is key in computing  \eqref{eq:step:genlmo}, and properties of $z$ can be used to reveal the support of $s$ with respect to $\mP$.

\begin{property}[Support optimality condition]
\label{prop:support_opt}
If $p_i$ is in the support of $x^*$ a minimizer of \eqref{eq:main}, then   
\[
-\nabla f(x^*)^Tp_i = \sigma_\mP(-\nabla f(x^*)).
\]
\end{property}

\paragraph{Example: $\ell_1$ norm.}  Consider the problem 
\[
\minimize{x} \quad \underbrace{f(x) + \frac{1}{2}\|x\|_1^2}_{g(x)}.
\]
In this case, $\sigma_\mP = \|\cdot\|_\infty$ is the dual norm of $\kappa_\mP = \|\cdot\|_1$.
Then, by setting the optimality condition $0\in \partial g(x^*)$ and decomposing by index, at optimality
\[
\begin{cases}
(-\nabla f(x^*))_i = \|x^*\|_1  \,\sign(x^*_i) &\text{ if } x_i^* \neq 0, \\
(-\nabla f(x^*))_i  \in \|x^*\|_1\, [-1,1] &\text{ if } x_i^* = 0.
\end{cases}
\]
In words, the gradient of $f$ along a coordinate for which the optimal variable is nonsmooth with respect to $\kappa_\mP$ is allowed  ``wiggle room''; in contrast, if $g(x)$ is smooth in the direction of $x_i$ then the gradient is fixed. In terms of support recovery,  $\max_i |z^*_i| = \|x^*\|_1$ and additionally, if ${|z^*_i| < \|x^*\|_1}$ then it must be that $x_i^* = 0$.

More generally, visually,   the condition $p_i^Tz^* = \sigma_\mP(z^*)$ says that at the optimum, the gradient in the direction of $p_i$ is as steep as allowable; $x^*$ wants to keep going in this direction, but is  blocked because of a constraint or nonsmooth penalty. For gauges, this non-smoothness only happens when the contribution of $p_i$ in $x$ is 0, thus translating to a support recovery property.

The proof follows from convex analysis principles describing the dual behaviors of $\kappa_\mP(x^*)$ and $\sigma_\mP(-\nabla f(x^*))$.
The property itself serves as the main principle behind dual screening methods; by identifying $p_i$'s that are sufficiently far from the maximum value, we can guess that such $p_i$'s do not appear in the support of $x^*$.


\paragraph{Noncompact $\mP$.} In practice, recession directions in $\mP$ may be desirable to allow for unpenalized directions. For example, in the total variation norm, which promotes smoothness, $\kappa_\mP(x) = 0$ if $x = \beta\mb 1$. In this case, a finite $\sigma_\mP(z)$ constrains $z$ to be in the nullspace of all such recession directions. In gCGM,  such gauges  are problematic because the solution to the generalized subproblem \eqref{eq:step:genlmo} is unbounded if $\nabla f(x)^Tc \neq 0$ for any $c$  in a recession direction. Therefore  we assume $\mP$ to be compact.

\paragraph{$0$ on the boundary of $\mP$.} It may be desirable to have $\kappa_\mP$ partially enforce conic constraints as well, such as in semidefinite optimization where $\kappa_\mP(X) = \tr(X) + \iota_{\cdot \succeq 0}(X)$ promotes low-rank positive semidefinite matrices. In this case, since no negative definite elements are in $\mP$, $0$ must be on the boundary of $\mP$.  In the dual, this corresponds to a recession direction, as any negative definite matrix $Z$ necessarily has $\sigma_\mP(Z) = 0$. 
This scenario does not affect the effectiveness nor analysis of gCGM; in particular, if \eqref{eq:step:genlmo} ever returns $s = 0$, then optimality is achieved.

\paragraph{Infinite atomic sets.} We assume that $\mP_0$ is a finite set. In low-rank matrix completion, for which CGMs are frequently used, the nuclear norm acts as the gauge function over the set of rank-1 norm-1 matrices, which is a compact but uncountably infinite set. In fact, the gCGM is still well-defined in this case, and all of the results in this paper are consistent. However,  since as there are no isolated points in $\mP_0$, it is impossible to guarantee finite-time exact support recovery (and in fact $\delta$ defined below is always 0). Thus, although safe screening rules do apply in this case, without modification they may not provide  practical advantages.

Gauges and support functions for convex sets are fundamental objects in convex analysis, and are discussed more by \citet{rockafellar1970convex,borweinlewis,freund1987dual,friedlander2014gauge}.

\subsection{Duality}
The Fenchel dual of \eqref{eq:main} can be computed as
\begin{equation}
\maximize{z\in\R^d} \quad -f^*(-z) - \phi^*(\sigma_\mP(z)),
\label{eq:main-dual}
\end{equation}
where for any convex function $f$, its convex conjugate is $f^*(z) = \sup_{x \in \R^d} z^Tx - f(x)$. 
Given $f$ differentiable, at optimality ${z^* = -\nabla f(x^*)}$.

For any $x,z\in \R^d$, the duality gap
\[
\gap(x,z) := f(x) + \phi(\kappa_\mP(x)) + f^*(-z) + \phi^*(\sigma_\mP(z)) 
\]
is nonnegative and $0$ only at optimality. 
Since at optimality $z^* = -\nabla f(x^*)$, a reasonable measure of suboptimality for a nonoptimal $x$ is $\gap(x,-\nabla f(x))$. 
In particular, 
\[
\gap(x,-\nabla f(x)) = (-\nabla f(x))^T(s-x) + \phi(\kappa_\mP(x)) -\phi(\kappa_\mP(s))
\]
can be used as a computable residual measure for both convergence tracking and screening rules; here, 
\[
s = \argmin{s' \in \R^d}\;\nabla f(x)^Ts' + \phi(\kappa_\mP(s'))
\]
is already computed in each step of the gCGM. 
When ${\phi = \iota_{\cdot \leq 1}}$ (the vCGM case) the gap calculation is much simpler, reducing to
\[
\gap(x,-\nabla f(x)) = (-\nabla f(x))^T(s-x).
\]

\subsection{Generalized CGM (gCGM)}

There are many ways of solving problems of the form \eqref{eq:main}, and our dual screening results and manifold identification results are in fact method-agnostic. 
Here, we investigate the gCGM, which  has almost as cheap of a per-iteration cost as the vCGM.
In particular, if we decompose $s$ in terms of its gauge value $\xi$ and normalized direction $\hat s$, then their minimizations can be done independently. 
Explicitly, step~\eqref{eq:step:genlmo} can be summarized in two steps, with $s = \xi\cdot\hat s$, and
\begin{equation}
\begin{array}{rcl}
\hat s &=& \LMO_\mP(-\nabla f(x)),\\
\xi&=& \argmin{\xi \geq 0}\; -\xi \cdot \sigma_\mP(-\nabla f(x)) + \phi(\xi),
\end{array}
\label{eq:genlmo}
\end{equation} 
where $\LMO_\mP(z) := \argmax{s\in \mP}\; x^Tz$ is the usual linear maximization oracle (LMO). For any compact set $\mP$, the LMO returns a finite $\hat s$; however, the minimization for $\xi$ is more complicated.  As a simple example, consider gCGM applied to the one-dimensional problem 
\[
\minimize{x} \quad \underbrace{\frac{1}{2}(x-c)^2}_{f(x)} + \underbrace{|x|}_{h(x)}.
\]
At the very first step, $f'(0) = -c$, and if $|c| > 1$ then $\xi$ is unbounded. Therefore, further conditions on $\phi$ must be imposed.

\subsection{Generalized penalty} 
The function $\phi:\R_+\to\R_+$ facilitates the transition of~\eqref{eq:main} from penalized to constrained optimization. When $\phi(\xi) = \xi$, then \eqref{eq:main} is a typical sparse regularized problem; at the other extreme, $\phi(\xi) = \iota_{\xi \leq C}$ an indicator function can constrain $x\in C\mP$, reducing everything to the vCGM case (vanilla CGM).

\begin{assumption}[Well-defined gCGM]
\label{asspt:phi1}
The function  ${\phi:\R_+\to\R_+}$  is monotonically nondecreasing over all $\xi \geq 0$.
Moreover, the set of subdifferentials of $\phi$ is not upper bounded:
\begin{equation}
\sup\;\{\alpha : \alpha \in \partial \phi(\xi)\} \; \overset{\xi\to+\infty}{\to}  \; + \infty.
\label{eq:subdiff_upper}
\end{equation}

\end{assumption}
\begin{assumption}[Convergence of gCGM]
\label{asspt:phi2}
The function  $\phi:\R_+\to\R_+$ is lower bounded by a quadratic function
\begin{equation}
\phi(\xi) \geq \mu_\phi \xi^2-\phi_0,
\label{eq:asspt-quadgrowth}
\end{equation}
for some $\mu_\phi > 0$ and  $\phi_0$. 
\end{assumption}

\begin{property}[Well-defined and converging gCGM]
\label{prop:phicond}
{$\;$}
Assumption \ref{asspt:phi1} ensures that the conjugate function
\begin{equation}
\phi^*(\nu) := \sup_{\xi \geq 0} \; \nu\xi - \phi(\xi)
\label{eq:phiconjdef}
\end{equation}
is finite-valued and attained for all $\nu \geq 0$. Moreover, there always exists a finite maximizer $\xi$.
 
Assumption \ref{asspt:phi2} further ensures that the derivative of $\phi^*$ is asymptotically nonexpansive; e.g. for some finite-valued $\xi_0$,
\[
(\phi^*)'(\nu) \leq \frac{\nu}{\mu_\phi}+\xi_0.
\]
\end{property}

\paragraph{Example: Monomials.}
For $1 \leq \alpha,\beta \leq +\infty$, the following $\phi:\R_+\to\R_+$ and $\phi^*:\R_+\to\R_+$ form a conjugate pair:
\[
\phi(\xi) = \frac{1}{\alpha} \xi^{\alpha}, \qquad \phi^*(\nu) = \frac{1}{\beta}\nu^\beta, \qquad \frac{1}{\alpha} + \frac{1}{\beta} = 1.
\]
In particular, in the case that $\alpha = 1$, then $\beta \to +\infty$, and the function 
\[
\phi^*(\nu) = \lim_{\beta\to+\infty} \frac{1}{\beta}\nu^\beta = 
\begin{cases}
0, & \nu \leq 1\\
+\infty, & \nu > 1.
\end{cases}
\]
As shown earlier, when $\alpha = 1$ then whenever $\nu > 1$ then $\phi^*(\nu) = +\infty$; we exclude this case as gCGM will not converge in this case. 
When $\alpha \geq 2$, $\phi$ is strongly convex and we can show $O(1/t)$ convergence of gCGM.
When  $1 < \alpha < 2$, $\phi^*(\nu)$ is finite and the iterates are well-defined, but  the method may converge or diverge.


 \paragraph{Example: Barrier functions.}
Consider 
\begin{equation}
\phi(\xi) = -\frac{1}{\beta}\log(C-\xi) - \frac{\xi}{C\beta} + \frac{\log(C)}{\beta},
\label{eq:logbarrier}
\end{equation}
which is a log-barrier penalization function for $\xi \leq C$; as $\beta \to +\infty$, $\phi(\xi)$  approaches the indicator function for this constraint. Its conjugate is 
\[
\phi^*(\nu) = C\nu   - \beta^{-1} \log(C\beta\nu+1),
\]
achieved at $\xi =C^2\beta\nu/(C\beta \nu+1)$. For all $C > 0, \beta > 0$, and $\nu\neq -(C\beta)^{-1}$, both $\phi^*$ and $\xi^*$ exist and are finite. Note also the implicit constraint, as $\phi(\kappa_\mP(x))$ is finite only if $x\in C\mP$.

\subsection{Generalized smoothness}
\begin{definition}
A function $f:\R^d\to\R$ is  $L$-smooth 
 with respect to $\mP$ if for all $x,y \in \R^d$:
\begin{equation}
f(x) - f(y) \leq (\nabla f(y))^T(x-y) + \frac{L}{2}\kappa_{\mP}(x-y)^2.
\label{eq:ass:smoothness}
\end{equation} 
\end{definition}
The purpose of this generalized notion is that sometimes, given the data, tighter bounds can be computed~\citep[see, e.g.,][]{nutini2015coordinate}.
\paragraph{Example: Quadratic function.} Suppose that 
\[
f(x) = \frac{1}{2} \|Ax\|_2^2 + b^Tx.
\]
Then 
\[
L = 
\begin{cases}
L_1 := (\max_i \|A_{:,i}\|_2)^2, & \kappa_\mP = \|\cdot\|_1\\
L_2 := \|A\|_2^2, & \kappa_\mP = \|\cdot\|_2\\
L_\infty := (\sum_i \|A_{:,i}\|_2)^2, & \kappa_\mP = \|\cdot\|_\infty.
\end{cases}
\]
While norm bounds would give  ${d^2L_{1} \geq  dL_{2} \geq  L_{\infty}}$, the actual values in $A$ might lead to tighter inequalities.
\paragraph{Example: Linear model.} Suppose that 
\[
f(x) = \frac{1}{n}\sum_{i=1}^n g(a_i^Tx),
\]
for some convex, smooth twice-differentiable function $g$ (e.g., logistic or exponential regression). 
Then 
\[
L = \left(\sup_{w \in \R} g''(w)\right)\left( \sup_{v\in \mP}\|Av\|_2^2\right).
\]

\paragraph{Equivalence to usual smoothness.} Suppose that $f$ is $L_2$-smooth in the usual sense (with respect to $\|\cdot\|_2$).  Then since $\diam(\mP)\kappa_\mP \geq \|x\|_2$, it follows that $L \leq \diam(\mP) L_2$. In this way, we refine the analysis of gCGM by absorbing the usual ``set size'' term into $L$, which in certain cases may be much smaller than $\diam(\mP) L_2$.

\subsection{Invariance}

One appealing feature of the vCGM is that the iteration scheme and analysis can be done in a way that is 
invariant to both linear scaling and translation. Specifically, if $\mQ = A\mP + b$, and $f(x) = g(Ax+b)$, then the two problems 
\[
\minimize{x\in \mP} \; f(x), \qquad \minimize{w\in \mQ} \; g(w)
\]
are equivalent. However when the gauge function is not used as an indicator, this translation invariance vanishes; in general, $\kappa_\mP(x) \neq \kappa_{\mP+\{b\}}(x+b)$. Therefore the generalized problem formulation \eqref{eq:main} is only linear (not translation) invariant; thus our analysis only maintains this invariance as well.

\begin{property}[Invariance]
\label{prop:invariance}
Consider two equivalent problems where $f(x) = g(Ax)$ and $\mQ = A\mP$:
\[
\mathrm{(P1)} \quad \minimize{x} \quad f(x) + \phi(\kappa_\mP(x)), 
\]
\[\mathrm{(P2)} \quad \minimize{w} \quad g(w) + \phi(\kappa_\mQ(w)).
\]
For any $x$, $w = Ax$,
\begin{itemize}
\item  $x$ optimizes (P1) $\iff$ $w$ optimizes (P2),
\item 
$\kappa_\mP(x) = \kappa_\mQ(w)$,

\item $\sigma_{\mP}(-\nabla f(x))  = \sigma_{\mQ}(-\nabla g(w)$,
\item 
$\LMO_\mQ(-\nabla g(w)) = A\;\LMO_\mP(-\nabla f(x))$,
\item 
 $f$ is $L$-smooth with respect to $\mP$ if and only if $g$ is $L$-smooth with respect to $\mQ$, 

\item  and $\gap(x,-\nabla f(x)) = \gap(w,-\nabla g(w))$. 
 \end{itemize}
\end{property}

\section{Main results}
\label{sec:mainresults}

In this section we give the main theoretical contributions: convergence rate, dual screening rule, and support identification complexity. These results all derive from some simple observations:
\begin{itemize}
\item The minimum duality gap at $x^{(t)}$ converges to 0 as $x^{(t)}\to x^*$ an optimal primal variable.  
\item The gradient error can be upper bounded by the gap, and support recovery is guaranteed when it is smaller than a problem-dependent constant,  which is difficult to compute in practice.
\item Without knowing this constant, one can still give partial support guarantees, which is used to construct screening rules.
\end{itemize}
We now state these points formally; all proofs are given in the appendix.

\begin{theorem}[Convergence]
\label{th:convergence}
Suppose that $x^{(t)}$ are the iterates of gCGM for which $f$ is $L$-smooth  with respect to $\widetilde \mP:=\mP\cup -\mP$,  $\phi:\R_+\to\R_+$ is monotonically nondecreasing, and satisfies Assumptions \ref{asspt:phi1} and \ref{asspt:phi2} for some $\mu_\phi > 0$.  Take $\theta^{(t)} = 2/(t+1)$. Then 
\[
 f(x^{(t)}) - f(x^*)  = O(1/t),
\]
and
\[
\min_{i\leq t} \gap(x^{(i)},-\nabla f(x^{(i)})) = O(1/t).
\]
\end{theorem}

A key difference between this result and previous works is that we do not assume or enforce bounded iterates.

The scaled gradient error will serve as our primary ``residual quantity'' in measuring distance to support recovery:
\[
\res(x) := \sigma_{\widetilde\mP}(\nabla f(x)-\nabla f(x^*)),
\]
and the symmetrization $\widetilde\mP := \mP \cup -\mP$ ensures that ${\sigma_{\widetilde\mP}(z-z^*)=\sigma_{\widetilde\mP}(z^*-z)}$, bounding errors in both directions. 
\begin{lemma}[Gap bounds residual]
\label{lem:gapboundsres}
For any primal feasible variable $x$, 
\[
\res(x) \leq \sqrt{L\, \gap(x,-\nabla f(x))}.
\]
\end{lemma}





Figure \ref{fig:resproof} gives a cartoon intuition as to what a small residual buys us. In particular, if $\delta$ is larger than $2\res(x)$, then a maximal element of $\{-\nabla f(x^*)^Tp_k\}_k$ must also be a maximal element of $\{-\nabla f(x)^Tp_k\}_k$. Since we can observe a bound on $\res(x)$, it is now possible to exclude which atoms are definitively \emph{not} in $\supp_\mP(x^*)$.
\begin{figure}
\begin{center}
\includegraphics[width=4in]{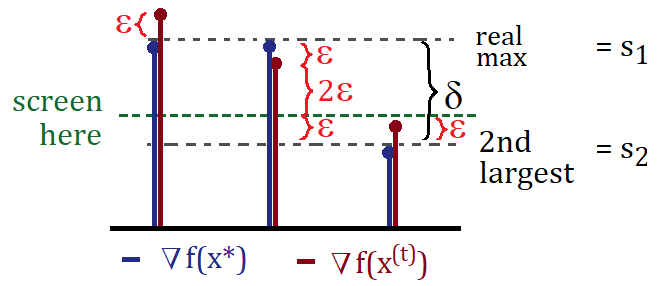}
\end{center}
\caption{\textbf{Support recovery.}
The constant $\delta$ differentiates maximal in-support values from the largest non-support value, as in \eqref{eq:defdelta}. $\epsilon = \res(x^{(t)})$ for some current (non-optimal) iterate $x^{(t)}$. 
Denote $z^* = -\nabla f(x^*)$ and $z^{(t)} = -\nabla f(x^{(t)})$.
Suppose $\sigma_\mP(z^{(t)}) = \sigma_\mP(z^*) + \epsilon$ (illustrated as $s_1+\epsilon$), its largest possible value. Then it is possible that some $p\in \supp_\mP(x^*)$ exists where $p^Tz^{(t)} = \sigma_\mP(z^{(t)}) - 2\epsilon$; thus, a safe screening rule can at largest be a threshold at $\sigma_\mP(z^{(t)})-2\epsilon$. This rule eliminates all  false negatives. To ensure no false positives, the largest possible non-optimal non-support value ($s_2+\epsilon$) must be smaller than the screened point. This can only happen if $\delta > 4\epsilon$.}
\label{fig:resproof}
\end{figure}

\begin{theorem}[Dual screening]
\label{th:screening}
Assume that $f$ is $L$-smooth with respect to $\widetilde \mP$. Then
for any $x$, any $p\in \mP_0$, 
\begin{equation}
\sigma_\mP(-\nabla f(x)) + p^T\nabla f(x) > 2\sqrt{L\gap(x,-\nabla f(x))}
\label{eq:gaprule}
\end{equation}
implies that $p\not\in \supp_\mP(x^*)$,
where $x^*$ is the optimal variable in \eqref{eq:main}.
\end{theorem}
A formal proof is in the appendix, following the logic in Figure \ref{fig:resproof}.

This gives rise to a dynamic screening rule: fix $\mS^{(0)} = \mP_0$ and
\[
\mS^{(t)} = \mS^{(t-1)} \setminus \{p : p \mathrm{\; satisfies  \;  \eqref{eq:gaprule} \; for\;  } x=x^{(t)}\}.
\]
\begin{remark}[Practical considerations]
Some things to note about this screening method:
\begin{itemize}
    \item Computing $L$ may be challenging, depending on $\kappa_\mP$; as shown previously, at the very least it may require a full pass over the data. However, this is a one-time calculation per dataset, and can be estimated if data are assumed to be drawn from specific distributions (as in sensing applications). 
    \item If $\mP_0$ is large (such as in submodular optimization) then checking condition \eqref{eq:gaprule} for each atom at each iteration is also cumbersome. However, if the screening is aggressive, then after a few iterations, the list of potential atoms to check will decrease quickly as well. 
    \item Computing the gap, in comparison, is almost automatic in gCGM, given that $z^{(t)} = -\nabla f(x^{(t)}$ is the (always feasible) dual candidate and $s^{(t)}$ already computed. In comparison, when dealing with a different dual candidate, then the term $f(x) + f^*(-z)$ is not easily upper bounded, and depending on the choice of $f$ may be difficult to compute in practice.
\end{itemize}
\end{remark}

The ``safeness" 
of the screening rule (Theorem \ref{th:screening}) ensures that $\mS^{(t)}\supseteq \supp_\mP(x^*)$, for all $t$. 
For support identification, we would like to find a $\bar t$ where for all $t > \bar t$, $\mS^{(t)} = \supp_\mP(x^*)$. Note that with a deterministically decaying sequence for $\theta^{(t)}$, finite-time support recovery \emph{without} screening is impossible, since any erroneously selected atoms early on can never fully diminish.
Even with screening, it is still not automatically guaranteed that such a finite $\bar t$ exists, since the problem itself may be degenerate \cite{lewis2011identifying,hare2011identifying,burke1988identification}. This occurs when $\delta = 0$, where
\begin{equation}
\delta:=\min_{p\not\in\supp_\mP(x^*)} \sigma_\mP(-\nabla f(x^*)) +(\nabla f(x^*))^Tx^*
\label{eq:defdelta}
\end{equation}
is a problem-dependent (algorithm-independent) quantity.

\begin{theorem}[Support identification of screened gCGM]
\label{th:supportid}
Assume $f$ is $L$-smooth with respect to $\widetilde\mP$. Then ${\mS^{(t)} = \supp_\mP(x^*)}$) when
\begin{equation}
 \sqrt{L\min_{i\leq t}\gap(x^{(i)},\nabla f(x^{(i)}))} <  \delta/4,
\label{eq:suppID:bound}
\end{equation}
which, under the assumptions of Theorem \ref{th:convergence}, happens at a rate $t = O(1/(\delta^2))$.
\end{theorem}


The proof follows from the gap bound (Lemma \ref{lem:gapboundsres}), gap rate (Theorem \ref{th:convergence}), and scrutiny of Figure \ref{fig:resproof}; specifically, when $\epsilon < \delta/4$, then any rule that screens away elements that are more than $2\epsilon$ from $\sigma_\mP(x^{(t)})$ will screen away \emph{all} the non-support elements.

\begin{remark}[Generality]
Note that Theorems \ref{th:screening} and \ref{th:supportid} impose no conditions on the sequence $\theta^{(k)}$, or choice of $\phi$, $f$, etc.,  except $L$-smoothness of $f$. In other words, for any method where $\epsilon(t)  \geq \min_{i \leq  t}\gap(x^{(i)},\nabla f(x^{(i)}))$ is known, then a corresponding screening rule and support identification rate automatically follow.
\end{remark}

\section{Experiments}
\label{sec:experiments}
We consider  sparse logistic regression
\begin{equation}
\min_{x\in \R^d} \;\; -\frac{1}{n}\sum_{i=1}^n \log(1+e^{b_ia_i^Tx} ) + \underbrace{\lambda\;\phi(C^{-1}\|x\|_1)}_{h(x)},
\label{eq:mnist-logreg}
\end{equation}
where $\lambda > 0$ controls the weighting of the penalty term and $C > 0$  the magnification of $\mP$. Here, $a_i\in \R^d$ are data vectors and $b_i \in \{-1,1\}$ are binary labels.
In all cases we run gCGM with $\theta^{(t)} = 2/(t+1)$.

\subsection{Synthetic experiments}
First, we generate  $a_i\in \R^{50}$ i.i.d.~standard Gaussian normal vectors, and fix $b_i = 1$, for $i = 1,...,100$, and analyze the numerical behavior of gCGM when $\phi(\xi) = p^{-1} \xi^p$; we fix $C = 1$ here.
The duality gap for different choices of $p$ are plotted in Figure \ref{fig:gap}, with $\lambda = 0.01$. For low values of $p$ we observe some numerical instability in the early iterates, as for ``flatter" penalty functions the new steps $s^{(t)}$ can be very large. Figure \ref{fig:res} compares different problem residuals with $p = 2$, $\lambda = 1.0$. 
In particular we are able to verify our $O(1/t)$ bound on all residuals, though it is clear that for this example, the gap is converging much more slowly than the gradient error $\sigma(z^{(t)}-z^*)^2$, which is almost twice as fast, which is why our screening rule, though safe, can be pessimistic in practice. Finally, Figure~\ref{fig:screening} shows the evolution of the support size for $p = 2$ and varying values of $\lambda$. In general, a larger value of $\lambda$ causes aggressive screening early on, while for larger values of $\lambda$ screening may be much slower, despite arriving at about the same final sparsity level.

\begin{figure}
\begin{center}
\includegraphics[width=4in]{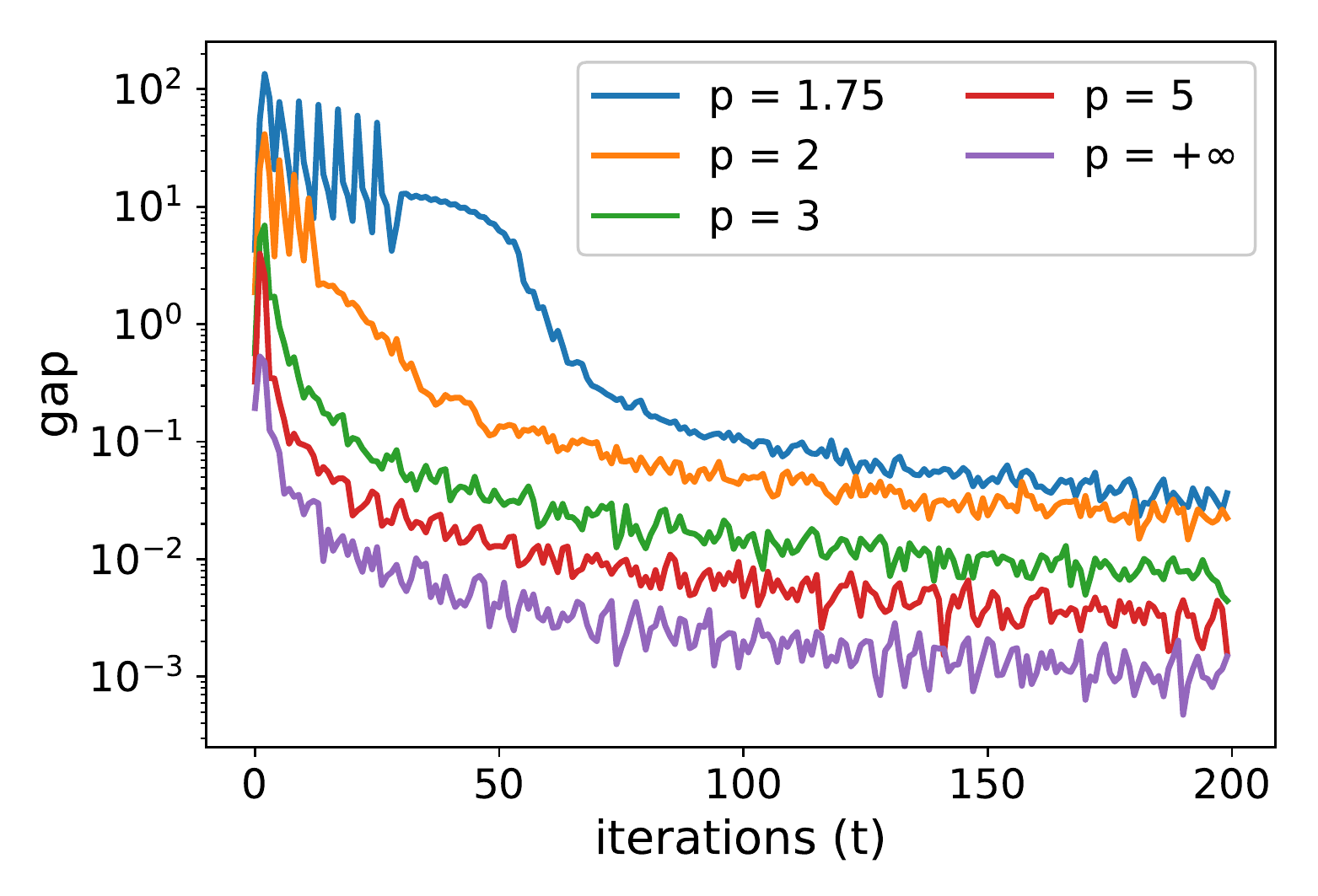}
\end{center}
\caption{\textbf{Duality gap for varying $p$.}  $\lambda = 0.01$. For ${p < 1.5}$, the method diverged.     $p = +\infty$  corresponds to vCGM. 
}
\label{fig:gap}
\end{figure}

\begin{figure}
\begin{center}
\includegraphics[width=4in]{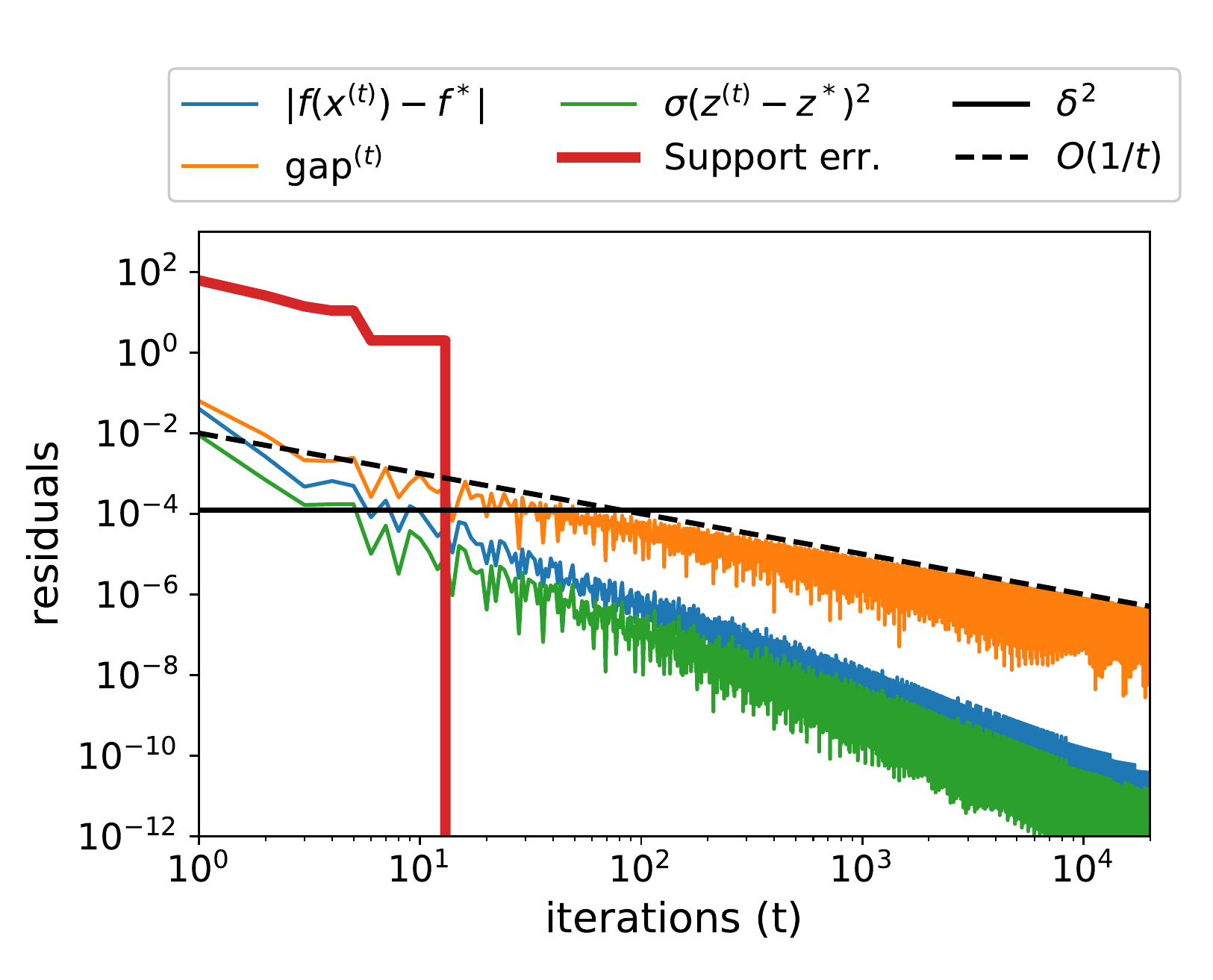}
\end{center}
\caption{\textbf{Residuals.}  $p = 2$, $\lambda = 1.0$. The objective error and gap decay at the computed rate of $O(1/t)$. The gradient error decays as $O(1/\sqrt{t})$. In fact when the gradient error $\sigma(z^{(t)}-z^*)$ dips below $\delta/4$, support error is 0; unfortunately, the gap takes longer to reach this point.}
\label{fig:res}
\end{figure}

\begin{figure}
\begin{center}
\includegraphics[width=4in]{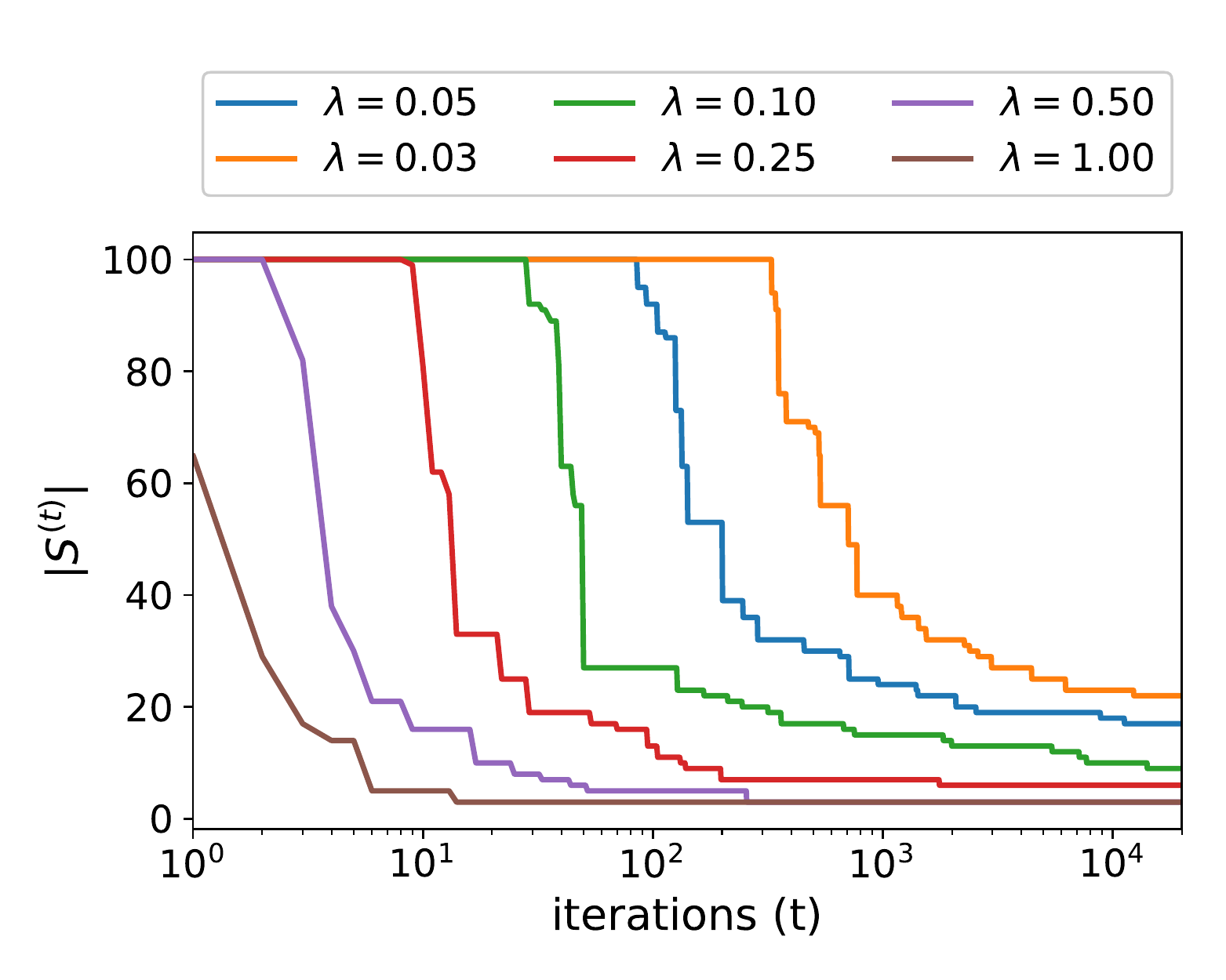}
\end{center}
\caption{\textbf{Screening.}  $p = 2$, and we plot $|\mS^{(t)}|$ the number of  unscreened variables at each iteration. We observe more aggressive screening for larger $\lambda$. 
}
\label{fig:screening}
\end{figure}

\subsection{MNIST classification of 4's vs 9's}
Figure \ref{fig:mnist} shows the screening behavior of \eqref{eq:mnist-logreg} on the binary classification problem of disambiguating 4's and 9's in the MNIST handwriting dataset. We experiment with three  schemes: one-norm squared regularization ($h(x) = \frac{\lambda}{2} \|x\|_1^2$), one-norm ball constraint ($h(x) = \iota_{C\mP}(x)$), and log barrier \eqref{eq:logbarrier}. 
All experiments are halted at 10,000 iterations for fair comparison.

There are two major observations.
First, the yellow curve (observed sparsity) is often much lower than the red curve (guarantee-able sparsity). This is because when $\lambda$ is small or $C$ is big, the gap converges slowly, and the condition $\gap(x^{(t)}) < \delta/4$ requires $t > 10,000$ (our stopping condition).
However, that is the tradeoff required for ``safety". 

Second, the red curve (guarantee-able sparsity) is only small when the blue curve (misclassification rate) is higher, suggesting an inherent performance/sparsity tradeoff. This tradeoff is in fact observed for all three choices of $\phi$ 
and suggests that in general, the MNIST classification task performs best without extreme sparsity.


\begin{figure}
\begin{center}
\includegraphics[width=4in]{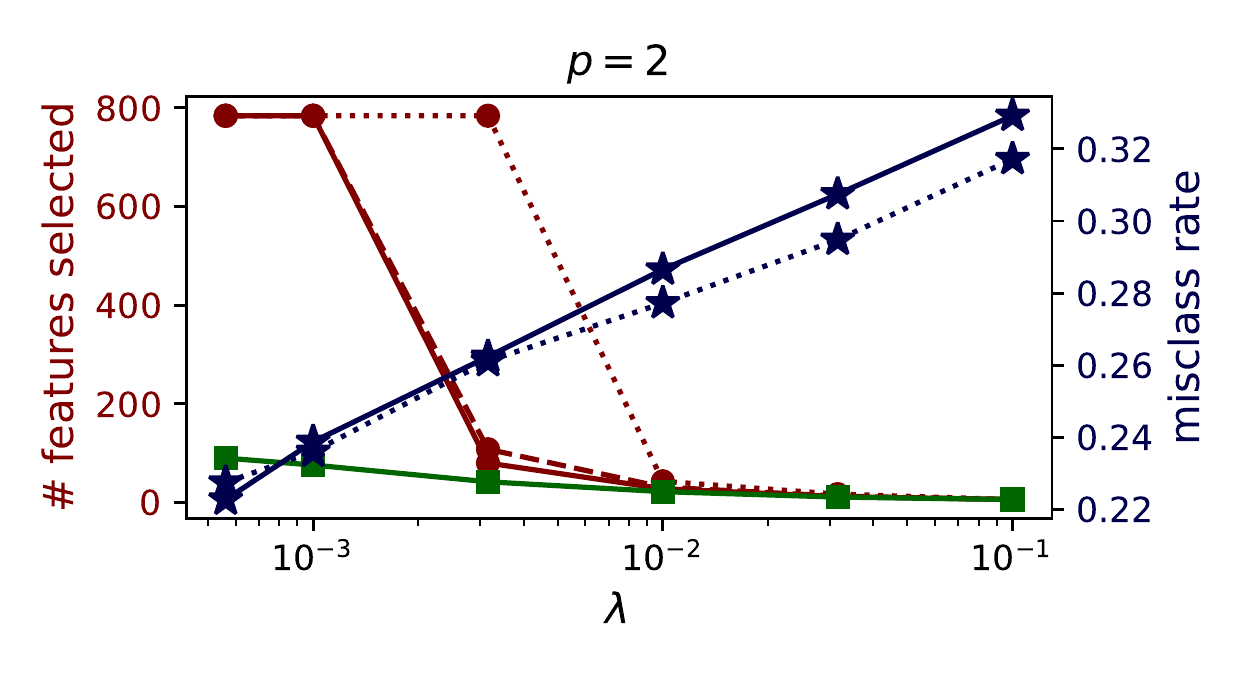}
\includegraphics[width=4in]{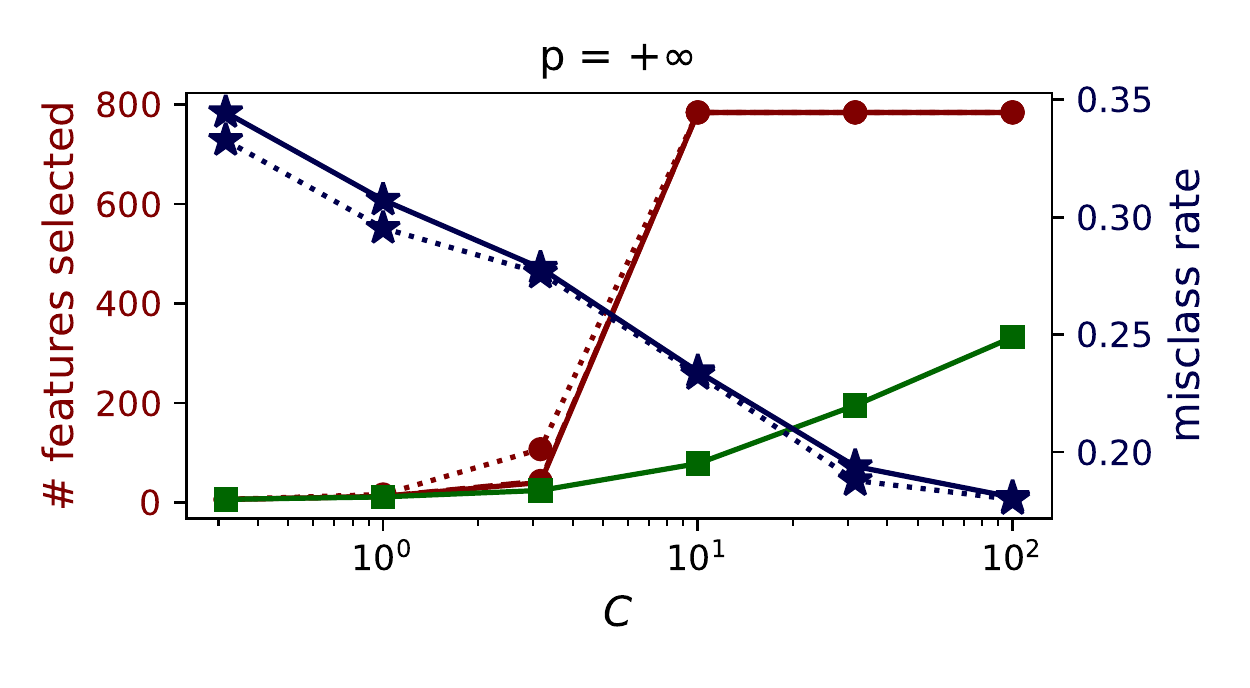}
\includegraphics[width=4in]{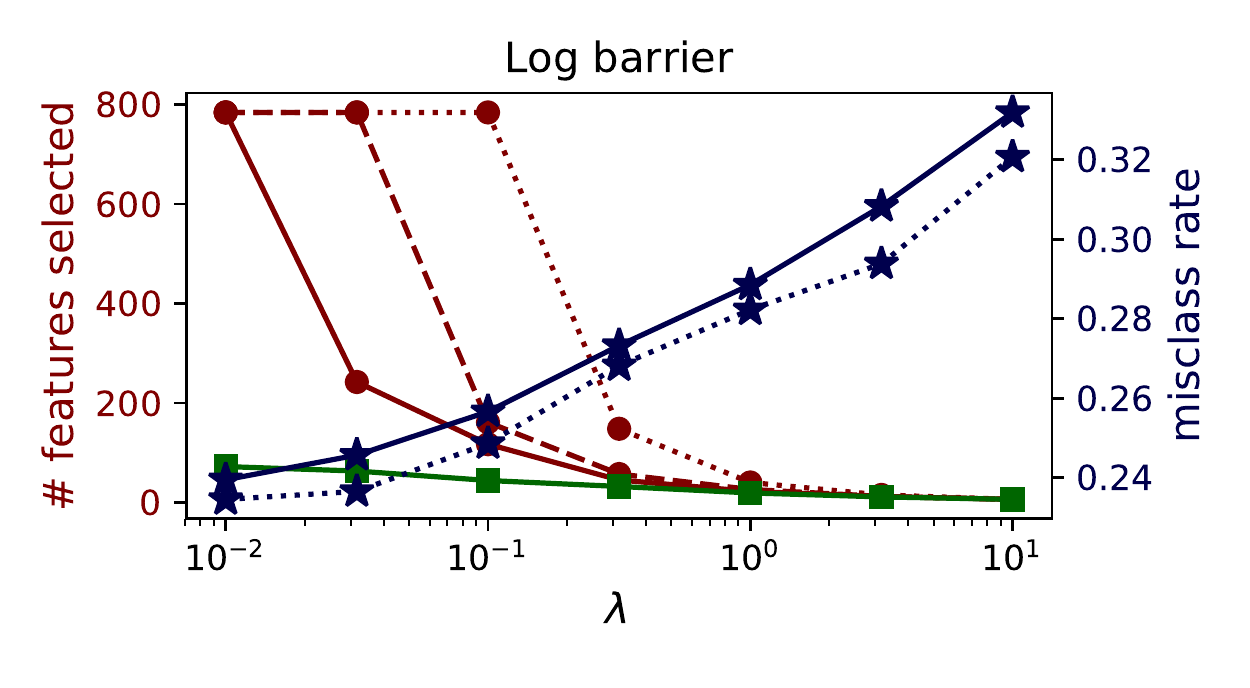}
\end{center}
\caption{\textbf{MNIST experiment.} Solid/dashed blue lines are train/test misclassification rates. Solid/dashed/dotted red lines are number of unscreened features at 10000 / 5000 / 1000 iterations; it is possible that   more features would be screened away after more iterations, as the gap converges very slowly for small $\lambda$.  Green square line plots the number of nonzeros of $x^{(10000)}$,  which  is observed to be stable. (\textbf{Top}) $h(x) = \frac{\lambda}{2}\|x\|_1^2$.  (\textbf{Middle}) $h(x) = \iota_{C\mP}(x)$. (\textbf{Bottom}) Log barrier function \eqref{eq:logbarrier} 
 where $C = 10$. 
 }
\label{fig:mnist}
\end{figure}

\section{Conclusion}
We have given a gap-based safe screening rule for a family of sparse optimization problems, for various types of sparse penalties and atoms. We analyze this in the context of the gCGM, and give rates for convergence and support identification for nondegenerate problems. In particular, the generalization over atom type and choice of $\phi$ allows for a much richer collection of sparse models,  interpolating between the piece-wise linear unconstrained LASSO penalty and the hard norm ball constraint. These penalties differ in their sensitivity toward hyperparameters, and may be more suited to a wider range of applications. 

A key promise in these rules is that, in the spirit of \citet{ghaoui2010safe}, screening is \emph{safe}, e.g., no true nonzero will be wrongly called a zero.  However, in practice this rule may be pessimistic, first because the gap may serve as an overly pessimistic upper bound of the gradient error, and second because sparsity in the true solution of the optimization problem may be overkill for sparsity of a solution that generalizes well for the machine learning task.  

Still,  there are practical advantages.
A sparsity guarantee  gives  storage benefits; a  model trained on a large server can be moved to a mobile device, for example, with no need for heuristic thresholding or rounding. 
And, if $\mP_0$ is very large, then screening can greatly improve the runtime of the linear minimization oracle (LMO), used in each step; since the rules are safe, this can be done without disrupting any convergence guarantees.

\begin{appendices}

\section{Helpful facts}
\label{appx:helpful}

\begin{lemma}[Relationship of $\kappa_\mP$ to $\|\cdot\|_2$]
Denote 
\[
\diam(\mP) := \sup_{x\in \mP,y\in \mP}\|x-y\|_2.
\] 
Then for any closed convex $\mP$,
\[
\diam(\mP)  \kappa_\mP(x) \geq \|x\|_2.
\]
\end{lemma}
\begin{proof}
Using another  classical definition for gauge functions,
\begin{eqnarray*}
\kappa_\mP(x) &=& \inf\{\mu \geq 0 : x\in \mu \mP\}\\
&\overset{\mP \subset \mathcal B_r}{\geq}& \inf \{\mu : x \in \mu \mathcal B_r\} = \inf \{\mu : x \in \mu r \mathcal B_1\} =  r^{-1}\|x\|_2,
\end{eqnarray*}
where $\mB_r$ is the smallest Euclidean ball of radius $r$ that includes $\mP$; that is, $r \leq \diam(\mP)$.
\end{proof}

We denote the subdifferential of a convex function $f$ at $x$ as $\partial f(x)$, and the normal cone of $\mP$ at $x$ as $\mN_\mP(x)$; See \citet{rockafellar1970convex}.

\begin{lemma}[Conjugate of nested function]
If $g(x) = \phi(\kappa_\mP(x))$ and $\phi$ is monotonically nondecreasing, then $g^*(z) = \phi^*(\sigma_\mP(z))$.
\end{lemma}
\begin{proof}
From the definitions, we have:
\begin{eqnarray*}
g^*(z) &=& \sup_{s} s^Tz - \phi(\kappa_\mP(s))\\
&=& \sup_{\xi, \hat s\in \mP} \xi \cdot \bar s^Tz - \phi(\xi)=  \phi^*\left( \sup_{ \hat s\in \mP}\bar s^Tz \right)= \phi^*(\sigma_\mP(z)).
\end{eqnarray*}
\end{proof}

\begin{lemma}[Chain rule for subdifferential (\citet{bauschke2011convex}, Corollary 16.72.)]
\label{lem:chainrulesubdiff}
Let $f:H\to \R$ be continuous and convex, and let $\phi:\R\to\R$ be increasing on $\range(f)$. Suppose that $(\ri(\range~f)+\R_{++}) \cap \ri\; \dom \;\phi \neq \emptyset$. Let $x\in H$ such that $\kappa(x)\in \dom(\phi)$. Then 
\[
\partial (\phi \circ f)(x) = \{\alpha u : \alpha \in \partial \phi(f(x)), \; u \in \partial f(x)\}.
\]
\end{lemma}


\begin{lemma}[Gap in primal form]
\label{lem:gapprimal}
For $f$ everywhere differentiable,
\[
\gap(x,-\nabla f(x)) = -\nabla f(x)^T(s-x) + h(x) - h(s),
\]
where
\[
s = \argmin{s} \; \nabla f(x)^Ts + h(s).
\]
\end{lemma}
\begin{proof}
By construction of $s$, $\nabla f(x)^Ts + h(s) = h^*(-\nabla f(x))$. And, in general, for convex lower semicontinuous $f$, $f(x) + f^*(\nabla f(x)) = x^T\nabla f(x)$. The rest follows from substitution.
\end{proof}

\section{Proofs from Section \ref{sec:prelim}}
\label{appx:sec2}

\begin{customprop}{\ref{prop:support_opt}}[Support optimality condition]
If $p_i$ is in the support of $x^*$ a minimizer of 
\[
\minimize{x} \; f(x) + \phi(\kappa_\mP(x))
\] 
where $f$ is everywhere differentiable and $\phi$ satisfies Assumptions \ref{asspt:phi1} and \ref{asspt:phi2}, 
then   ${-\nabla f(x^*)^Tp_i = \sigma_\mP(-\nabla f(x^*))}$.
\end{customprop}

\begin{proof}
Without loss of generality, we assume $0\in \mP$, since $\kappa_\mP = \kappa_{\mP \cup \{0\}}$. Denote  $z^* = -\nabla f(x^*)$. Now, applying Lemma~\ref{lem:chainrulesubdiff}, the optimality condition for \eqref{eq:main} is 
\begin{equation}
z^*  \in \alpha \partial   \kappa_\mP (x^*),
\label{eq:opthelp1}
\end{equation}
for some $\alpha \in \partial \phi(\xi)$ with $\xi = \kappa_\mP(x^*)$. Since $\phi$ is monotonically nondecreasing over $\R^+$, $\alpha \geq 0$. If $\alpha = 0$ then the property is trivially true. Now consider $\alpha > 0$.
Noting that $\kappa_\mP = \sigma_{\mP^\circ}$ where $\mP^\circ$ is the polar set of $\mP$, 
\begin{eqnarray*}
\alpha^{-1}z^* &=& \argmax{z\in \mP^\circ}\;(x^*)^Tz \\
\iff (z^*)^Tx^* 
&=& \kappa_{\mP^\circ}(z^*)\sigma_{\mP^\circ}(x^*) =  \kappa_\mP(x^*)\sigma_\mP(z^*).
\end{eqnarray*}
 Now take the conic decomposition $x^* = \sum_{i=1}^m c_i p_i$ where $c_i \geq 0$, and 
\[
 (x^*)^Tz^*  = \sum_{i=1}^m c_i p_i^Tz^* \leq \underbrace{\left(\sum_{i=1}^m c_i \right)}_{=\kappa_\mP(x^*)}\underbrace{( p_i^Tz^*)}_{\leq \sigma_\mP(z^*)},
\]
which is with equality if and only if $p_i^Tz^*= \sigma_\mP(z^*))$ whenever $c_i > 0$. 

\end{proof}

\begin{customprop}{\ref{prop:phicond}}[Well-defined and converging gCGM]
{$\;$}
\begin{itemize}
\item 
Assumption \ref{asspt:phi1} ensures that the conjugate function
\begin{equation}
\phi^*(\nu) := \sup_\xi \; \nu\xi - \phi(\xi)
\label{eq:phiconjdef2}
\end{equation}
is finite-valued and attained for all $\nu \geq 0$. Moreover, there always exists a finite maximizer $\xi$.

\item 
Assumption \ref{asspt:phi2} further ensures that the derivative of $\phi^*$ is asymptotically nonexpansive; e.g. 
\[
(\phi^*)'(\nu) \leq \frac{\nu}{\mu_\phi}+\xi_0
\]
for some finite-valued $\xi_0$.

\end{itemize}
\end{customprop}

\begin{proof}
{$\;$}
\begin{itemize}
\item  \textbf{Assumption 1.}
Since $\phi$ has nonempty domain, ${\phi}^*(\nu) > -\infty$ for all $\nu$. It can be shown that ${\phi}^*(\nu) < +\infty$ whenever there exists a finite $\xi \geq 0$ where $\nu\in \partial {\phi}(\xi)$, since then 
\[
\phi^*(\nu) = \underbrace{\xi\nu}_{\text{finite}} - \underbrace{\phi(\xi)}_{\geq \phi(0)}.
\]

Now define $\mS := [{\phi}'(0),+\infty)$. By the assumptions on $\phi$, for any $\nu \in \mS$, there exists some finite $\xi \geq 0$ where $\nu \in \partial \phi(\xi)$.

Now consider $\nu\in [0,\phi'(0))$. By convexity,  for all $\xi \geq 0$,
\[
\phi(\xi)-\phi(0) \geq \phi'(0)\xi \geq 0,
\]
and thus for all such $\nu$, $\nu \in \partial \phi(0)$.

Therefore there always exists a finite maximizer $\xi$ of \eqref{eq:phiconjdef2}; since also  $\phi^*(\nu)$ is not $\pm \infty$, then \eqref{eq:phiconjdef2} is always attained.


\item \textbf{Assumption 2.} 
Assume that $\phi_0$ is as large as possible; e.g., there exists some finite $\xi_0$ where $\phi(\xi_0) = \mu\xi_0^2 + \phi_0$. Then for all $\xi \geq \xi_0$, for all $\nu\in \partial \phi(\xi)$,
\[
\mu (\xi^2-\xi_0^2 )\leq \phi(\xi)-\phi(\xi_0) \leq \nu(\xi-\xi_0)
\]
and therefore
\[
\nu\geq\mu \frac{(\xi+\xi_0)(\xi-\xi_0)}{\xi-\xi_0}  = \mu \xi+\mu \xi_0 \; \iff \; \xi \leq  \mu^{-1} \nu -  \xi_0.
\]
Therefore, for any $\xi$, any $\nu\in \partial \phi( \xi)$ must satisfy 
\[
\xi \leq \max\{\xi_0,  \mu^{-1} \nu -  \xi_0\} \leq   \mu^{-1} \nu +\xi_0.
\]
By Fenchel Young, this must apply to all $\xi \in \partial \phi^*(\nu)$, which completes the proof.

\end{itemize}
\end{proof}

\begin{customprop}{\ref{prop:invariance}}[Invariance]
Consider two equivalent problems where $f(x) = g(Ax)$ and $\mQ = A\mP$:
\[
\mathrm{(P1)} \quad \minimize{x} \quad f(x) + \phi(\kappa_\mP(x)),
\]
\[
\mathrm{(P2)} \quad \minimize{w} \quad g(w) + \phi(\kappa_\mQ(w)).
\]
For any $x$, $w = Ax$,
\begin{itemize}
\item  $x$ optimizes (P1) $\iff$ $w$ optimizes (P2).

\textbf{Proof.} This follows from verifying the optimality conditions.

First,  since $\kappa_\mP(x) = \kappa_\mQ(w)$ (next bullet), $\partial \phi(\kappa_\mP(x)) = \partial \phi(\kappa_\mP(w))$. 

Second, taking $w'=Ax'$ for any $x'$,  $s\in \partial \kappa_\mQ(w)\iff $
\[
 \forall w', \; \underbrace{\kappa_\mQ(w) - \kappa_\mQ(w') }_{=\kappa_\mP(x)-\kappa_\mP(x')}\leq s^T(w-w') = (A^Ts)^T(x-x'),
\]
and thus
$A^T\partial \kappa_\mQ(w)  =  \partial \kappa_\mP(x) \cap \range(A^T)$.

Finally, note that $\nabla f(x) = A^T\nabla g(w)$. Therefore, picking some $\alpha \in \partial \phi(\kappa_\mP(x))$,
\[
0\in A^T\nabla g(w) + \alpha A^T\partial \kappa_\mQ(w)  
\]
is equivalent to 
\[
0\in  \nabla f(x) + \alpha \partial \kappa_\mP(x),
\]
since $\nabla f(x) \in \range(A^T)$.
\item 
$\kappa_\mP(x) = \kappa_\mQ(w)$ 

\textbf{Proof.}
The gauge definition can also be written as 
\[
\kappa_\mP(x) = \inf\{\mu \geq 0 : x\in \mu\mP\}.
\]
From this, it can be seen that 
\[
\inf\{\mu \geq 0 : x\in \mu\mP\} = \inf\{\mu \geq 0 : Ax  \in \mu A\mP \}.
\]

\item $\sigma_{\mP}(-\nabla f(x)) - b^T\nabla g(w) = \sigma_{\mQ}(-\nabla g(w)$ 

\textbf{Proof.} Take $z = -\nabla f(x)$ and $v = -\nabla g(w)$. Then
\begin{eqnarray*}
\sigma_\mP(z) &=& \sigma_\mP(-A^Tv) \\
&=& \sup_{s\in \mP}s^T(A^Tv) = \sup_{s\in \mP}(As)^Tv =  \sup_{s\in A\mP+b} s^Tv=  \sigma_{\mQ}(v).
\end{eqnarray*}

\item 
$\LMO_\mQ(-\nabla g(w)) = A\;\LMO_\mP(-\nabla f(x)) $

\textbf{Proof.} Take $z = -\nabla f(x)$ and $v = -\nabla g(w)$. Then

\begin{eqnarray*}
\LMO_\mQ(v)&=&\argmax{s\in A\mP}\;v^Ts \\
&=& A \left(\argmax{s\in \mP}\;v^T(As)\right) = A\left(\argmax{s\in \mP}\;(A^Tv)^Ts\right) = A\;\LMO_\mP(z) 
\end{eqnarray*}

\item 
 $f(x) = g(Ax+b)$ is $L$-smooth and $\mu$-strongly convex with respect to $\mP$ iff $g$ is $L$-smooth and  $\mu$-strongly convex with respect to $ \mQ$
 
\textbf{Proof.}
 This follows from the observation that
 \[
 f(x) - f(x') - \nabla f(x')^T(x-x') 
 =g(w) - g(w') - \nabla g(w')^T(w-w') ,
 \]
and from $\kappa_\mP(x) = \kappa_\mQ(w)$ .

\item  $\gap(x,-\nabla f(x)) = \gap(w,-\nabla g(w))$. 

\textbf{Proof:}
\begin{eqnarray*}
u &:=& \argmin{u} \; \nabla g(w)^Tu + \phi(\kappa_\mQ(u)) = 
\argmin{u} \; \nabla f(x)^T(Au) + \phi(\kappa_\mP(Au)) ,
\end{eqnarray*}
and thus
\[ s := \argmin{s} \; \nabla f(x)^Ts + \phi(\kappa_\mP(s)) = Au.
\]
The rest follows from \eqref{lem:gapprimal} and noting that 
\[
-\nabla f(x)^T(s-x) +\phi(\kappa_\mP(x)) - \phi(\kappa_\mP(s))
= -\nabla g(w)^T(u-w) + \phi(\kappa_\mQ(w)) - \phi(\kappa_\mP(u)).
\]
 \end{itemize}
\end{customprop}

\section{Generalized smoothness}
\label{appx:gensmooth}

The following bound holds for any closed convex $\mP$, which may or not be compact or symmetric.
\begin{lemma}[Smoothness equivalences]
Suppose that  $f$ is $L$-smooth with respect to $\kappa_\mP$:
\begin{equation}
f(y)-f(x) \leq \nabla f(x)^T(y-x) + \frac{L}{2} \kappa_\mP(x-y)^2.
\label{eq:smoothness:app}
\end{equation}
Then the following also holds:
\begin{enumerate}
\item Expansiveness
\begin{eqnarray}
(\nabla f(x)-\nabla f(y))^T(x-y)\geq   \frac{1}{2L}(\sigma_\mP(\nabla f(x)-\nabla f(y))^2  + \sigma_\mP(\nabla f(y)-\nabla f(x))^2),
\label{eq:expansiveness:app}
\end{eqnarray}
\item Strongly convex conjugate
\begin{eqnarray}
f(y)-f(x) &\geq& \nabla f(x)^T(y-x) +  \frac{1}{2L}\sigma_\mP(\nabla f(y)-\nabla f(x))^2.
\label{eq:stronglyconvex:app}
\end{eqnarray}
\end{enumerate}
\end{lemma}

\begin{proof}
The proof largely follows from \cite{nesterovlectures}, mildly adapted.
\begin{itemize}
\item First prove \eqref{eq:smoothness:app} $\Rightarrow$  \eqref{eq:expansiveness:app}.
Construct $g(x) = f(x) - x^T\nabla f(y)$, which is convex, also $L$-smooth, and has minimum at $x = y$.
Then, for any $w$, 
\[
g(y) \leq g(x+w) \overset{(a)}{\leq} g(x) + \nabla g(x)^Tw + \frac{L}{2}\kappa_\mP(w)^2,
\]
where (a) is since $g$ is $L$ smooth and convex.

Now pick  
\[
w \in \frac{1}{L}\sigma_\mP(-\nabla g(x))\partial \sigma_\mP(-\nabla g(x)),
\]
which implies
\begin{eqnarray*}
\frac{L}{\sigma_\mP(-\nabla g(x))}  w & \in& \argmax{\kappa_\mP(u)\leq 1}\;\langle u, -\nabla g(x)\rangle = \partial \sigma_\mP(-\nabla g(x)),
\end{eqnarray*}
and thus
\[
\kappa_\mP(w) = \frac{\sigma_\mP(-\nabla g(x))}{L},
\]
and
\[
\langle w, -\nabla g(x)\rangle = \frac{1}{L}\sigma_\mP(-\nabla g(x))^2.
\]
Then
\[
\frac{L}{2}\kappa_\mP(w)^2 = \frac{1}{2L} \sigma_\mP(-\nabla g(x))^2,
\]
and
plugging in the construction for $g$ gives
\begin{eqnarray*}
g(y) -g(x)&\leq&  \underbrace{\nabla g(x)^Tw + \frac{L}{2}\kappa_\mP(w)^2}_{-\frac{1}{2L}\sigma_\mP(-\nabla g(x))^2}\\
\iff f(y) -f(x) &\leq&  (y-x)^T\nabla f(y)  - \frac{1}{2L} \sigma_\mP(\nabla f(y)-\nabla f(x))^2.
\end{eqnarray*}
Applying the last inequality twice gives
\begin{eqnarray*}
(y-x)^T(\nabla f(y)-\nabla f(x))\leq  \frac{1}{2L} ((\sigma_\mP(\nabla f(x)-\nabla f(y))^2  +(\sigma_\mP(\nabla f(y)-\nabla f(x))^2).
\end{eqnarray*}

\item Now prove \eqref{eq:smoothness:app} $\Rightarrow$  \eqref{eq:stronglyconvex:app}.
Using the same $g$ as before, consider
\[
\min_z\; g(x)+\langle \nabla g(x), z-x\rangle + \frac{L}{2}\kappa_\mP(x-z)^2
=\min_w\; \langle \nabla g(x), w\rangle + \frac{L}{2}\kappa_\mP(w)^2.
\]
Using optimality conditions, picking $w = z-y$, we have
\[
0\in \nabla g(x) + L \kappa_\mP(w)\partial \kappa_\mP(w)\iff -\frac{1}{L\kappa_\mP(w)}\nabla g(x) = \argmax{\sigma_\mP(u)\leq 1}\langle u,w\rangle ,
\]
which implies
\[
\sigma_\mP(-\nabla g(x)) = L\kappa_\mP(w), \qquad  -\frac{1}{L\kappa_\mP(w)}\langle w,\nabla g(x)\rangle = \kappa_\mP(w).
\]
so
\[
\langle w, -\nabla g(x)\rangle = L\kappa_\mP(w)^2 = \frac{1}{L}\sigma_\mP(-\nabla g(x))^2,
\]
and overall
\[
g(y) \geq \min_z\; g(x)+\langle \nabla g(x), z-x\rangle + \frac{L}{2}\kappa_\mP(x-z)^2 = g(x) - \frac{1}{2L}\sigma_\mP(-\nabla g(x))^2.
\]
Plugging in $f$ gives
\[
f(y)-f(x) \geq (y-x)^T\nabla f(y) - \frac{1}{2L}\sigma_\mP(\nabla f(y)-\nabla f(x))^2.
\]

\end{itemize}

\end{proof}

\begin{corollary}[Uniqueness of gradient] If \eqref{eq:smoothness:app} holds and $0\in\inte~\mP$, then $\nabla f(x)$ is unique at the optimum.
\end{corollary}
\begin{proof}
Assume that $f(x)=f(x^*)$ for some $x\neq x^*$, $x$ feasible. Then by optimality conditions, \\
${\nabla f(x^*)^T(x^*-x)\leq 0}$, and thus
\[
\underbrace{f(x)-f(x^*)}_{=0} \geq \underbrace{\nabla f(x^*)^T(x-x^*)}_{\geq 0} + \frac{1}{2L}\sigma_\mP(\nabla f(x)-\nabla f(x^*))^2,
\]
which implies that $\sigma_\mP(\nabla f(x)-\nabla f(x^*)) = 0$. Since $0\in \inte~\mP$, this can only happen if $\nabla f(x) = \nabla f(x^*)$.
\end{proof}

\begin{lemma}[Hessian sufficient condition]
For some closed convex set $\mP$, and some convex twice differentiable function $f:\R^n\to\R$, suppose that 
\[
L = \sup_{p_1,p_2\in \mP} p^T\nabla^2 f(x)p.
\]
Then 
\[
f(x) - f(y) - \nabla f(y)^T(x-y) \leq \frac{L}{2}\kappa_{\mP}(x-y)^2.
\]
\end{lemma}

\begin{proof}
By definition and positive homogeneity of $\kappa_\mP(x)$, more generally 
\[
L \geq \kappa_\mP(u)\kappa_\mP(v) u^T\nabla^2f(x) v,\;\forall u,v,x.
\]
Then
\begin{eqnarray*}
D_f(y||x) &=& f(x)-f(y)-\nabla f(y)^T(x-y) \\
&=& \int_0^1  (\nabla f(x+(y-x)t)-\nabla f(x))^T(y-x)dt\\
&=& \int_0^1 \int_0^t (y-x)^T\nabla^2 f(x+(y-x)s) (y-x)\;ds \;dt \leq L\kappa_\mP(y-x)^2 \int_0^1\int_0^t ds dt =  \frac{L}{2}\kappa_\mP(y-x)^2.
\end{eqnarray*}
\end{proof}

\footnote{See also \citet[Appendix A]{mirrokni}.}

\begin{lemma}[Gradient suboptimality bound]
\label{lem:gradsubopt}
Suppose that $x^* = \argmin{x \in \R^d}\;f(x)+h(x)$ where $f$ is $L$-smooth with respect to $\sigma_\mP$ and $h$ is convex. Then 
\[
f(x)-f(x^*)+ h(x)-h(x^*)\geq \frac{1}{L}\sigma_{\widetilde\mP}(\nabla f(x)-\nabla f(x^*))^2.
\]
\end{lemma}

\begin{proof}
For any $\alpha\in \partial h(x)$ and $\alpha^*\in \partial h(x^*)$,
\begin{eqnarray*}
\frac{1}{L}\sigma_{\mP}(\nabla f(x)-\nabla f(x^*))^2 &\overset{(a)}{\leq}& (\nabla f(x)-\nabla f(x^*))^T(x-x^*) + {\underbrace{(\alpha+\alpha^*-\alpha-\alpha^*)}_{0}}^T(x-x^*)\\
&=& (\nabla f(x)+\alpha)^T(x-x^*) +
{\underbrace{(-\nabla f(x^*)-\alpha^*)}_{=0}}^T(x-x^*) +
\underbrace{(\alpha^*-\alpha)^T(x-x^*)}_{\leq 0}\\
&\overset{(b)}{\leq}& f(x)-f(x^*) + h(x)-h(x^*);
\end{eqnarray*}
we derive (a) from expansiveness, and (b) from convexity of $f+h$.
\end{proof}

 \section{Proofs from Section \ref{sec:mainresults}}
\label{appx:sec3}

The following Lemma will be used in computing the objective value bound.
\begin{lemma}[One step value bound] 
\label{lem:onestepbound}
For $f$ $L$-smooth with respect to $\kappa_\mP$ and $\phi$ $\mu$-convex, 
\begin{equation}
g(x^{(t+1)})-g(x^{(t)}) \leq -\theta^{(t)} \gap(x^{(t)},\nabla f(x^{(t)})) + \frac{(\theta^{(t)})^2}{2}(6L(\sigma(-\nabla f(x^*)+\nu_0)^2 +3 L^2 \Delta^{(t)} + 3 L^2 {\bar\Delta}^{(t-1)})
\label{eq:lem:onestepbound}
\end{equation}
where 
\[
g(x):= f(x) + \underbrace{\phi(\kappa_\mP(x))}_{h(x)}, \quad \Delta^{(t)} = g(x^{(t)})-g(x^*),
\]
we take the sequence $\theta^{(t)} = 2/(t+1)$, 
and $\bar \Delta^{(t)}$ represents an averaged suboptimality:
\begin{equation}
\sqrt{\bar\Delta^{(t)}} = \left(\sum_{u=1}^t u\right)^{-1}\left(\sum_{u=1}^t u \sqrt{\Delta^{(u)}}\right).
\end{equation}
\end{lemma}
\begin{proof}
For one step, at $x=x^{(t)}$, define 
\begin{eqnarray*}
s &:=& \argmin{s}\; s^T\nabla f(x) + \phi(\kappa_\mP(s))\\
x^+ &:=& (1-\theta)x + \theta s,
\end{eqnarray*}
and $\Delta = \Delta^{(t)}$.
 Since, by smoothness of $f$,
\begin{eqnarray*}
f(x^+) - f(x) &\leq& \nabla f(x)^T(x^+-x) + \frac{L}{2}\kappa(x^+-x)^2\\
&=&\theta \nabla f(x)^T(s-x) + \frac{L\theta^2}{2}\kappa(s-x)^2,
\end{eqnarray*}
and, by convexity of $h$,
\[
h(x^+) = h((1-\theta) x + \theta s) \leq (1-\theta) h(x) + \theta h(s),
\]
then 
\[
g(x^+)- g(x) \leq \theta\underbrace{(\nabla f(x)^T(s-x) + h(s)-h(x))}_{=A} + \frac{L\theta^2}{2}\underbrace{\kappa(s-x)^2}_{=B}.
\]
\paragraph{Term $A$.}
By construction of $s$, 
\[
s^T\nabla f(x) + h(s) = h^*(-\nabla f(x)),
\]
and in general,
\[
f(x) + f^*(\nabla f(x)) =  x^T\nabla f(x),
\]
and therefore 
\[
A = \nabla f(x)^T(s-x) + h(s)-h(x) = -f(x)-h(x) - f^*(\nabla f(x) + h^*(-\nabla f(x)) = -\gap(x,-\nabla f(x)).
\]

\paragraph{Term $B$.} By convexity and homogeneity of $\kappa$,

\begin{eqnarray*}
\kappa(s^{(t)}-x^{(t)}) &\leq & \kappa(s^{(t)})+\kappa(x^{(t)})\\
&=&\kappa(s^{(t)}) + \kappa(\theta^{(t-1)}s^{(t-1)}+(1-\theta^{(t-1)})x^{(t-1)})\\
&\leq & \kappa(s^{(t)}) + \theta^{(t-1)}\kappa(s^{(t-1)})+(1-\theta^{(t-1)})\kappa(x^{(t-1)})\\
&\leq & \kappa(s^{(t)}) + \theta^{(t-1)}\kappa(s^{(t-1)})+(1-\theta^{(t-1)})\theta^{(t-2)}\kappa(s^{(t-2)})+(1-\theta^{(t-1)})(1-\theta^{(t-2)})\kappa(x^{(t-2)})\\
&\leq & \kappa(s^{(t)}) + \sum_{u=1}^{t-1} \kappa(s^{(u)})\theta^{(u)}\prod_{t'=u+1}^{t-1}(1-\theta^{(t')}).
\end{eqnarray*}

Taking $\theta^{(t)} = \frac{2}{t+1}$,  then 
\[
\theta^{(u)}\prod_{t'=u+1}^{t-1}(1-\theta^{(t')})= \frac{2 u}{t(t-1)},
\]
 so
\[
\kappa(s^{(t)}-x^{(t)}) \leq \kappa(s^{(t)}) + \frac{2}{t(t-1)}\sum_{u=1}^{t-1} \kappa(s^{(u)})u.
\]
By optimality conditions on the update for $s^{(t)}$ (\eqref{eq:genlmo} in main text), 
\[
\kappa(s) = \xi = \argmin{\xi} \; -\xi\cdot\sigma(z) + \phi(\xi) \iff \sigma(z) = \phi'(\xi) \iff \xi = (\phi^*)'(\sigma(z)) \overset{(a)}{\leq} \mu^{-1}\sigma(z)+\nu_0,
\]

where (a) follows from Assumption \ref{asspt:phi2}. 
Then
\[
\kappa(s) \leq \mu^{-1}\sigma(z)+\nu_0 \leq  \mu^{-1}\sigma(z^*) + \mu^{-1}\sigma(z-z^*)+\nu_0 \overset{(b)}{\leq} \mu^{-1}\sigma(z^*) +\nu_0+\mu^{-1} \sqrt{L\Delta},
\]
where (b) follows from Lemma \ref{lem:gradsubopt}.
Overall this gives 
\begin{eqnarray*}
\kappa(s^{(t)}-x^{(t)})^2 &\leq& \mu^{-2}\left(\sigma(z^*) +\nu_0+ \sqrt{L\Delta^{(t)}} +   (\sigma(z^*) +\nu_0)+
\sqrt{L} \left(\sum_{u=1}^{t-1}  u\right)^{-1}
\left(\sum_{u=1}^{t-1} u \sqrt{\Delta^{(u)}} \right)\right)^2\\
 &\overset{(c)}{\leq}& 6\mu^{-2}(\sigma(z^*)+\nu_0)^2 + 3\mu^{-2}L\Delta^{(t)} +3\mu^{-2}L
\left(\left(\sum_{u=1}^{t-1}  u\right)^{-1}
\left(\sum_{u=1}^{t-1} u \sqrt{\Delta^{(u)}} \right)\right)^2,
\end{eqnarray*}
where (c) comes from $(\sum_{i=1}^m c_i)^2\leq m\sum_{i=1}^m c_i^2$.
\end{proof}

\begin{lemma}[Objective value bound]
\label{lem:objvalbound}
Given $f$ is $L$-smooth with respect to $\widetilde\mP$ and $\phi:\R_+\to\R_+$ is monotonically increasing and $\mu$-strongly convex, then the objective error decreases as 
\[
g(x^{(t)})-g(x^*) = O(1/t).
\]
\end{lemma}
\begin{proof}
Take $\bar t > 12B$ large enough so that for all $t \geq \bar t$, $\frac{3L^2}{2\mu^2}(\theta^{(t)})^2\leq  \theta^{(t)}/3$. Then define 
\[
A = \frac{3L}{\mu^2}\sigma(-\nabla f(x^*)+\nu_0)^2, \qquad B = 3L^2\mu^{-2}.
\]
Then, using Lemma \eqref{lem:onestepbound}, we have
\[
\Delta^{(t+1)}-\Delta^{(t)} \leq -\frac{1}{2}\theta^{(t)}\Delta^{(t)} + (\theta^{(t)})^2(A+B\bar\Delta^{(t)}).
\]
We now pick $G$ large enough such that for all $t \leq \bar t$, $\Delta^{(t)} \leq G/t$, and $G > 24A$. Since $\Delta^{(t)}$ is always a bounded quantity, this is always possible. Then, for all $t < \bar t$, 
\[
\sqrt{\bar \Delta^{(t)}} \leq \frac{\sqrt{G}}{t(t+1)}\sum_{t'=1}^t \sqrt{t'}\overset{(a)}{\leq} \frac{2\sqrt{G}}{3t(t+1)}t^{3/2},
\]
where (a) is by  integral rule, and so
\[
\bar \Delta^{(t)} \leq  \frac{4Gt}{9(t+1)^2} \leq \frac{G}{2t}.
\]
Now we make an inductive step. Suppose that for some $t$, $\Delta^{(t')} < G/t'$ for all $t' \leq t$. Pick $\theta^{(t)} = 2/(t+1)$. Then 
\begin{eqnarray*}
\Delta^{(t+1)} &\leq& \Delta^{(t)}-\frac{2}{3}\theta^{(t)}\Delta^{(t)} + (\theta^{(t)})^2(A+B\bar\Delta^{(t)})\\
&\leq & \frac{G}{t} - \frac{2}{3} \frac{2G}{t+1}\frac{1}{t} + \frac{4}{(t+1)^2}\left(A + \frac{GB}{2t}\right)\\
&= & \frac{G}{t+1}\left(\frac{t+1}{t} - \frac{4}{3t} + \frac{4A}{(t+1)G} + \frac{2B}{t(t+1)} \right)\\
&\leq & \frac{G}{t+1}\left(1 - \frac{1}{3t}  + \frac{4A}{tG} + \frac{2B}{t^2} \right)\\
&=& \frac{G}{t+1}\left(1 + \frac{1}{t} \left(-\frac{1}{3} + \underbrace{ \frac{4A}{G} }_{<1/6}+ \underbrace{\frac{2B}{t}}_{< 1/6}\right) \right) \leq \frac{G}{t+1},
\end{eqnarray*}
which satisfies the inductive step.
\end{proof}

The following is a generalized and modified version of a proof segment from \cite{jaggi2013revisiting}, which will be used for proving $O(1/t)$ gap convergence.
\begin{lemma}
\label{lem:gaphelper1}
Pick some $0 < T_2 < T_1$ and pick  
\[
\bar k = \ceil{D(k+D)/(D+T_1)}-D \;\Rightarrow\;  \frac{D}{D+T_1} \leq \frac{\bar k+D}{k+D} \leq \frac{D}{D+T_2}.
\]
Then if
\[
 \frac{C_1(D+T_1)}{D} \leq C_3\cdot \log\left(\frac{D+T_2}{D}\right),
 \]
  then for all $k > T_1$,
\[
 \left(\frac{C_1}{D+\bar k} + \sum_{i=\bar k}^{k} \frac{C_2}{(D+i)^2}-\frac{C_3}{D+i}\cdot \frac{1}{D+k}\right)  < 0.
\]
\end{lemma}
\begin{proof}
Using integral rule, we see that
\[
\sum_{i=\bar k}^{k} \frac{1}{(D+i)^2}\leq  \int_{z=\bar k -1}^{k-1} \frac{1}{(D+i)^2} = \frac{1}{D-1+ k}-\frac{1}{D-1+\bar k}
\]

\[
\sum_{i=\bar k}^{k} \frac{1}{D+i}\geq  \int_{z=\bar k }^{k} \frac{1}{D+i} = \log(D+k)-\log(D+\bar k).
\]
This yields 
\begin{eqnarray*}
c(k) &:=&  \frac{C_1}{D+\bar k} + \sum_{i=\bar k}^{k} \frac{C_2}{(D+i)^2}-\frac{1}{D+i}\cdot \frac{C_3}{D+k}\\
&\leq &  \frac{C_1}{D+\bar k} + \frac{C_2}{D-1+ k}-\frac{C_2}{D-1+\bar k} + \frac{C_3}{D+k}\cdot (\log(D+\bar k) - \log(D+k))\\
&\leq &  \frac{C_1(D+T_1)}{D(D+k)} + \underbrace{\frac{C_2}{D-1+ k}-\frac{C_2}{D-1+\bar k}}_{<0} + \frac{C_3}{D+k}\cdot \log\left(\frac{D}{D+T_2}\right)\\
&\leq &  \frac{C_1(D+T_1)}{D(D+k)} + \frac{C_3}{D+k}\cdot \log\left(\frac{D}{D+T_2}\right)< 0.
\end{eqnarray*}
\end{proof}

\begin{lemma}
[Generalized non-monotonic gap bound]
\label{lem:gapbound}
Given 
\begin{itemize}
\item 
$\Delta^{(t)} := g(x^{(t)}) - g(x^*) \leq \frac{G_1}{t+D}$ for some $G_1$, 
\item $\theta^{(t)} = \frac{G_2}{t+D}$ for some $G_2$ and $D$, and 
\item $\Delta^{(t+1)} - \Delta^{(k)}(1+\alpha \theta^{(k)}) \leq -\theta^{(k)}\gap(k) + (\theta^{(k)})^2 G_3$ for some $G_3$,
\end{itemize}
then for 
\[
G_4 \geq \frac{G_1}{G_2}  \frac{(D+2)}{D( \log\left(\frac{D+1}{D}\right) )} ,
\]
we have
\[
\min_{i\leq t}\gap^{(i)} \leq \frac{G_4}{t+D}.
\]
\end{lemma}
\begin{proof}  
 We have
\[
\Delta^{(t+1)}-\Delta^{(t)}\leq \alpha \theta^{(t)}\Delta^{(t)} - \theta^{(t)}\gap^{(t)} + G_3(\theta^{(t)})^2.
\] 
Now assume that for all $i \leq t$, $\gap^{(i)} > \frac{G_4}{t+D}$. Then, telescoping from $\bar t$ to $t$ gives 
\begin{eqnarray*}
\Delta^{(t+1)}  &\leq& \Delta^{(\bar t)} + \sum_{i=\bar t}^t \left( \alpha \theta^{(i)}\Delta^{(i)} - \theta^{(i)}\gap^{(i)} + G_3(\theta^{(i)})^2\right)\\
& < & \frac{G_1}{\bar t + D} + \sum_{i=\bar t}^t \left( \alpha \frac{G_1G_2}{(i+D)^2}- \frac{G_2}{i+D}\frac{G_4}{t+D} + \frac{G_3 G_2^2}{(i+D)^2}\right).
\end{eqnarray*}

Picking $C_1 = G_1$, $C_2 =\alpha G_1G_2+G_3G_2^2$, $C_3 = G_2G_4$, and invoking Lemma \ref{lem:gaphelper1}, this yields that $\Delta^{(t+1)} < 0$, which is impossible. Therefore, the assumption must not be true. 
\end{proof}

\begin{customthm}{\ref{th:convergence}}[Convergence]
Suppose that $x^{(t)}$ are the iterates of gCGM for which $f$ is $L$-smooth  with respect to $\widetilde \mP$,  $\phi:\R_+\to\R_+$ is monotonically increasing and $\mu$-strongly convex.  Take $\theta^{(t)} = 4/(t+2)$. Then 
\[
 f(x^{(t)}) - f(x^*)  = O(1/t)
\]
and
\[
\min_{i\leq t} \gap(x^{(i)},-\nabla f(x^{(i)})) = O(1/t).
\]
\end{customthm}
\begin{proof}
The proof follows from Lemma \ref{lem:objvalbound} and \ref{lem:gapbound}.
\end{proof}

\begin{customlem}{\ref{lem:gapboundsres}}[Gap bounds residual]
For any primal feasible variable $x$, 
\[
\res(x) \leq \sqrt{L\, \gap(x,-\nabla f(x))}.
\]
\end{customlem}

\begin{proof}
Taking $g(x) =  \phi(\kappa_\mP(x))$, we have
\[
g^*(z) = \sup_y\; y^Tz - \phi(\kappa_\mP(y)) \geq (x^*)^Tz - \phi(\kappa_\mP(x^*)).
\]
Additionally, by Fenchel-Young, $f(x) + f^*(-u) \geq -x^Tu$
Therefore 
\begin{eqnarray*}
\gap(x,u) &=& f(x) + \phi(\kappa_\mP(x)) + f^*(-u) + g^*(u)\\
&\geq & -(x-x^*)^Tu + \phi(\kappa_\mP(x))  - \phi(\kappa_\mP(x^*))\\
&\overset{\phi \text{ convex}}{\geq} & -(x-x^*)^T u + w^T (x-x^*)
\end{eqnarray*}
for some $w \in \partial (\phi \circ \kappa_\mP)(x^*))$.

Take $u = -\nabla f(x)$. Then 
\begin{eqnarray*}
\gap(x,-\nabla f(x)) &\geq& (x-x^*)^T( \nabla f(x) -\nabla f(x^*))  + \underbrace{(x-x^*)^T( \nabla f(x^*)+ w ) }_{\geq 0}\\
&\geq & \frac{1}{L_e}\sigma_{\widetilde\mP}(\nabla f(x^*)-\nabla f(x))^2.
\end{eqnarray*}
\end{proof}

\begin{customthm}{\ref{th:screening}}[Dual screening]
Assume that $f$ is $L$-smooth with respect to $\widetilde \mP$. Then
for any $x$, any $p\in \mP_0$, 
\begin{equation}
\sigma_\mP(-\nabla f(x)) + p^T\nabla f(x) > 2\sqrt{L\gap(x,-\nabla f(x))}
\label{eq:gaprule2}
\end{equation}
implies that $p\not\in \supp_\mP(x^*)$,
where $x^*$ is the optimal variable in \eqref{eq:main}.
\end{customthm}

\begin{proof}
From Lemma \ref{lem:gapboundsres}, we have that when condition \eqref{eq:gaprule2} holds,
\[
\sigma_\mP(-\nabla f(x)) + p^T\nabla f(x) > 2\sigma_{\widetilde\mP}(\nabla f(x)-\nabla f(x^*)).
\]
Then, by the triangle inequality, 
\[
\begin{split}
&\sigma_\mP(-\nabla f(x^*))+p^T\nabla f(x^*) \\
&\quad \geq\sigma_\mP(-\nabla f(x)) +p^T\nabla f(x) 
 -  2\sigma_{\widetilde\mP}(\nabla f(x)-\nabla f(x^*)) \\
&\quad > 0.
\end{split}
\]
Thus by Property \ref{prop:support_opt}, $p\not\in \supp_\mP(x^*)$.
\end{proof}

\end{appendices}

\newpage
\bibliography{refs}
\bibliographystyle{apalike}

\end{document}